\documentclass{article}

\usepackage[nonatbib,final]{neurips_2022}

\usepackage[utf8]{inputenc} 
\usepackage[T1]{fontenc}    
\usepackage{hyperref}       
\usepackage{url}            
\usepackage{booktabs}       
\usepackage{amsfonts}       
\usepackage{nicefrac}       
\usepackage{microtype}      
\usepackage{xcolor}         

\usepackage{amsmath}
\usepackage{amssymb}
\usepackage{mathtools}
\usepackage{amsthm}
\usepackage{bbm}
\usepackage{amsfonts}
\usepackage{mathrsfs}
\usepackage{acronym}
\usepackage{algorithm}
\usepackage{algpseudocode}
\usepackage{float}
\usepackage{soul}

\usepackage{wrapfig}
\usepackage{lipsum}

\theoremstyle{plain}
\newtheorem{theorem}{Theorem}[section]
\newtheorem{proposition}[theorem]{Proposition}
\newtheorem{lemma}[theorem]{Lemma}
\newtheorem{corollary}[theorem]{Corollary}
\theoremstyle{definition}
\newtheorem{definition}[theorem]{Definition}

\theoremstyle{remark}

\newtheorem*{theorem*}{Theorem}
\newtheorem*{corollary*}{Corollary}

\newacro{LHS}{left-hand side}
\newacro{RHS}{right-hand side}
\newacro{iid}[i.i.d.]{independent and identically distributed}
\newacro{MAPE}[MAPE]{Mean Absolute Percentage Error}
\newacro{MLE}[MLE]{Maximum Likelihood Estimator}

\DeclareMathOperator*{\argmax}{arg\,max}
\DeclareMathOperator{\Var}{Var}
\DeclareMathOperator{\Cov}{Cov}
\DeclareMathOperator{\KL}{KL}
\DeclareMathOperator{\UFM}{u^{M}}
\DeclareMathOperator{\Image}{Im}

\newcommand{\smax}[1][]{s^{\ast}_{#1}}
\newcommand{\ufi}[1][]{u_{#1}}
\newcommand{\UFI}[1][]{u^{\ast}_{#1}}
\newcommand{\UFIB}[1][]{u_B^{#1}}
\newcommand{\UFIind}[1][]{u_{I}^{#1}}
\newcommand{\EUFI}{\mathbb{E}[U]}

\newcommand{\indep}{\perp \!\!\! \perp}

\title{Bounding and Approximating Intersectional Fairness through Marginal Fairness}

\author{%
  Mathieu Molina \\
  Inria\\
  FairPlay team\\
  91120 Palaiseau, France\\
  \texttt{mathieu.molina@inria.fr} \\
   \And Patrick Loiseau \\
      Inria\\
      FairPlay team\\
      91120 Palaiseau, France\\
   \texttt{patrick.loiseau@inria.fr} 
}

\begin{document}

\maketitle

\begin{abstract}
  Discrimination in machine learning often arises along multiple dimensions (a.k.a. protected attributes); it is then desirable to ensure \emph{intersectional fairness}---i.e., that no subgroup is discriminated against. It is known that ensuring \emph{marginal fairness} for every dimension independently is not sufficient in general. Due to the exponential number of subgroups, however, directly measuring intersectional fairness from data is impossible. 
In this paper, our primary goal is to understand in detail the relationship between marginal and intersectional fairness through statistical analysis. We first identify a set of sufficient conditions under which an exact relationship can be obtained. Then, we prove bounds (easily computable through marginal fairness and other meaningful statistical quantities) in high-probability on intersectional fairness in the general case. Beyond their descriptive value, we show that these theoretical bounds can be leveraged to derive a heuristic improving the approximation and bounds of intersectional fairness by choosing, in a relevant manner, protected attributes for which we describe intersectional subgroups. Finally, we test the performance of our approximations and bounds on real and synthetic data-sets.
\end{abstract}

\vspace{-3mm}

\section{Introduction}

\vspace{-2mm}

Research on fairness in machine learning has been very active in recent years, in particular on fair classification under \emph{group fairness} notions, see e.g., \cite{equalopportunity, pmlr-v65-woodworth17a, zafar_disparate, Zafar17c, chouldechova2017fair, Romano20a}. Such notions define demographic groups based on so-called \emph{protected attributes} (e.g., gender, race, religion), and impose that some statistical quantity be constant across the groups. For instance, demographic parity imposes that the class-1 classification rate is the same for all groups, but other notions were defined such as equal opportunity \cite{equalopportunity} or calibration by group \cite{chouldechova2017fair}---see a survey in \cite{barocas-hardt-narayanan}. As exact fairness is too constraining, one often measures \emph{unfairness}, which roughly quantifies the distance to the fairness constraint. 

Most works on fair classification consider a single protected attribute and hence only two (or a small number of) groups. Then, they use measures of unfairness to evaluate and penalize classifiers in order make them more fair. This is making an implicit but very fundamental assumption that one can estimate the unfairness measure from the data at hand. With only a few groups, this assumption is indeed easily satisfied as there are sufficiently many data points for each group. 

In many---if not most---real-world applications, there are multiple protected attributes (typically 10-20) along which discrimination is prohibited \cite{equal_credit_opportunity,CodeTravail}. It is then desirable to consider the strong notion of \emph{intersectional fairness}, which roughly specifies that no subgroup (defined by an arbitrary combination of protected attributes) is unfavorably treated. In that case, however, estimating the unfairness measure becomes very challenging: as the number of groups is exponentially large (e.g., $2^{10}$ for 10 binary protected attributes), it is very likely that the dataset has at least one subgroup for which there are very few (or zero) data point. A potential solution is to treat each protected attribute separately through its \emph{marginal unfairness} (which is easy to estimate); but it was observed in several real-world and algorithmic examples that it is not sufficient to ensure intersectional fairness \cite{crenshaw,pmlr-v81-buolamwini18a,Kgerrymandering,Krich}. This raises the question: \emph{How to estimate intersectional fairness from data, and what is its precise relation to marginal fairness?}
To date, only very few papers have tackled this issue. \cite{Kgerrymandering, Krich} adopt a definition of intersectional fairness that weights the unfairness of each group by its size. This allows them to get large-samples generalization guarantees of empirical estimates (hence solving the estimation issue), but then it does not protect minorities since it allows a very high unfairness for tiny subgroups---which is contradictory to the intuitively desired behavior. 

\cite{pmlr-v80-hebert-johnson18a} makes a similar assumption by considering only subgroups above a minimum size, which eases estimate generalization.
\cite{Fintersectional} on the other hand uses the more natural definition of intersectional fairness based on the worst treated group irrespective of its size; but they consider only a few protected attributes, precisely to have enough data points on each subgroup to estimate intersectional unfairness. \cite{UFIbayesian} extends this work by proposing methods to interpolate for subgroups for which too few points are available, based on Bayesian machine learning models. However, this work is empirical and does not give any guarantee on the estimates obtained. 
In this paper, we also use the natural (strong) definition of intersectional fairness but we take instead a purely statistical approach. We view the protected attributes as random variables to understand intersectional fairness and how it related to marginal fairness more finely.

\textbf{Contributions:} We identify sufficient conditions under which intersectional fairness can be exactly derived from marginal densities, which clarifies when marginal unfairness is a good estimate of intersectional unfairness. We prove probabilistic bounds on intersectional unfairness based on marginal densities and independence measures of the protected attributes, that we show are easy to estimate. We propose a method to improve the approximation of intersectional unfairness and the theoretical bound based on grouping carefully some of the protected attributes together, which we do through a heuristic by leveraging the independence measures exhibited in our bounds. We perform experiments on real and synthetic datasets that illustrate the performance of our approach. In particular, we show that grouping with our heuristic does improve the approximation of intersectional unfairness. To the best of our knowledge, our work is the first work to exploit statistical information to better understand and estimate intersectional (un)fairness. Our work is fairly general and can be instantiated for a variety of standard fairness notions (demographic parity, equal opportunity, etc.). For simplicity, we focus on discrete protected attributes and on classification, but most of the core results can be extended to other cases.

\textbf{Further Related Works:} \cite{overlapping_groups} proposes a unified framework to train a fair classifier under intersectional fairness metrics, but without taking into account regimes with sparse group membership data. Some works propose to audit the accuracy of fairness metrics in contexts other than intersectionality, when there are missing data \cite{fairness_missing_data} or when there are unlabeled examples \cite{trust_fairness_metric}. Others tackle the problem of intersectionality beyond group fairness, e.g., \cite{causal_intersectionality}  considers causal intersectional fairness. Finally there has been some interest \cite{minimax_fairness,minimax_fairness_algo} in a different formulation of intersectional group fairness as a multi-objective optimization problem where each objective is the discrimination faced by a given protected group. Another interesting approach to fairness is individual fairness developed in \cite{individual_fairness}, however this is quite different from group fairness metrics on which we focus on and our techniques do not apply. 

\vspace{-3mm}

\section{Setting and Models}

\vspace{-2mm}

\subsection{Basic Setting}\label{sec:basic-setting}

\vspace{-2mm}

\emph{Notational convention}: Wherever useful, for any two random variables $V$ and $W$, we will use the short-hand $p_{V}(v)\!=\!\Pr(V\!=\!v)$, $p_{V,W}(v,w)\!=\!\Pr(V\!=\!v,W\!=\!w)$ and $p_{V\mid W}(v \mid w)\!=\!\Pr(V\!=\!v \mid W\!=\!w)$.

Consider a multi-class classification task. A given individual is described by a tuple of random variables $(X,A,Y)$ drawn according to a distribution $\mathcal{D}$ where $X$ is the features vector, $Y$ is the label with values in $\mathcal{Y}$, and $A=(A_1,...,A_d)$ is a $d$-tuple of protected attributes. The only variable used to make a prediction is $X$ and the only variable to measure unfairness is $A$, but otherwise there are no constraints and $A$ can be a part of $X$. We denote the support of $X$, $Y$, $A$ and $A_k$ for $1\leq k \leq d$, by $\mathcal{X}$, $\mathcal{Y}$, $\mathcal{A}$ and $\mathcal{A}_k$ respectively. We assume that $\mathcal{A}$ is finite (hence discrete). For a deterministic classifier $h$, $\hat{Y}=h(X)$ is the predicted class for a random individual. The classifier $h$ is fixed, as we are interested in measuring its fairness and not finding a fair classifier. 

To compare the discrimination between groups, we consider a second random variable $A'$ such that $(A' \mid \hat{Y})$ is \ac{iid} to $(A \mid \hat{Y})$.
Some authors look at the difference in the treatment of protected groups as a ratio \cite{Fintersectional}, and some others as a difference \cite{Kgerrymandering}. Here we choose to study discrimination in terms of ratio. We further apply a logarithm to symmetrize the discrimination measure between two protected groups and for ease of computation. We will consider Statistical Parity for simplicity of exposition, but other group fairness metrics can be either derived directly or adapted using the methods developed in this paper (see Appendix \ref{app.gen_other_fairness}). We define our measure of unfairness as follows:

\begin{definition}
For a distribution $\mathcal{D}$ and a classifier $h$, we define the \emph{intersectional unfairness} and the \emph{$k^{th}$ protected attribute marginal unfairness} as: \\ \vspace{-5mm}
\begin{align} \label{def.eq.UFI}
&\UFI= \sup_{(y,a,a')\in \mathcal{Y} \times \mathcal{A}^2} \ufi(y , a,a'), \quad \text{and} \quad \UFI[k]=\sup_{(y,a_k,a_k')\in \mathcal{Y} \times \mathcal{A}_k^2} \ufi[k](y , a_k,a_k') \\
\text{with} \ &\ufi (y , a,a') \!=\! \Big\vert \log \! \Big(\frac{\Pr(\hat{Y}\!=\!y \! \mid \! A\!=\!a)}{\Pr(\hat{Y}\!=\!y \! \mid \! A'\!=\!a')} \Big) \! \Big\vert ,\ \ufi[k] (y , a_k,a_k')\!=\!\Big\vert \log \! \Big(\frac{\Pr(\hat{Y}\!=\!y \! \mid \! A_k\!=\!a_k)}{\Pr(\hat{Y}\!=\!y \! \mid \! A'_k\!=\!a_k')}  \Big)\! \Big \vert.
\end{align}
\end{definition}
\vspace{-2mm}
One could think that if the marginal unfairness of each protected attribute is smaller than some $\epsilon>0$, then the overall unfairness is smaller than $\epsilon$; measuring $\UFM\!=\!\sup_k \UFI[k]$ corresponds to this idea. As stated in the introduction this is not sufficient to describe unfairness and we can still have $\UFI\!>\!\UFM$. We can rewrite (\ref{def.eq.UFI}) as $\UFI \!=\!\sup_{\mathcal{Y}}\! \log (\sup_{\mathcal{A}} p_{\hat{Y} \mid A}/\inf_{\mathcal{A}} p_{\hat{Y} \mid A} )$,
and similarly for $\UFI[k]$. This means that to measure unfairness we only need to analyze the function $p_{\hat{Y} \mid A}$.
\vspace{-3mm}
\subsection{Estimation of Unfairness}
\vspace{-1mm}
If we want to estimate unfairness, the most straightforward approach is to estimate the probability mass function $p_{A,\hat{Y}}$ and then to compute the unfairness over these estimated distributions. The main difficulty in estimating the unfairness is estimating $\inf p_{\hat{Y} \mid A}$, as we can upper bound the $\sup$ by $1$, but we cannot easily lower bound the $\inf$.  For a data-set of $n$ samples and $d$ protected attributes, we denote for $(a,y) \in \mathcal{A}\times \mathcal{Y}$ the counts by group and prediction as $N_{a,y}\!=\!\sum_{i=1}^n \mathbbm{1}[(A^{(i)},\hat{Y}^{(i)})\!=\!(a,y)]$ where $(A^{(i)},\hat{Y}^{(i)})$ is the $i$-th \ac{iid} realization of $(A,\hat{Y})$. The empirical probability is then defined as $\hat{\Pr}(\hat{Y}\!=\!y, A\!=\!a)\!=\!N_{a,y}/n$.
\cite{Fintersectional} shows in Theorem VIII.3 that the error made by using empirical estimates is decreasing in $N_a$, which means that there needs to be sufficient data for each protected group to estimate $\UFI$. 
When there are many protected groups the probability that at least one group receives no sample is high, and in this case there is at least one $a$ in $\mathcal{A}$ for which the empirical probability $\hat{\Pr}(\hat{Y}\!=\!y \mid A\!=\!a)$ is undefined, hence the $\inf$ and $\sup$ cannot be computed. \cite{UFIbayesian} and \cite{Fintersectional} alleviate this issue of $0$-counts by using a Dirichlet prior of uniform parameter $\alpha>0$. This yield the Bayesian estimates $(N_{a,y}+\alpha)/(n+\vert \mathcal{A} \vert \vert \mathcal{Y} \vert \alpha)$, that are then used to compute the estimator $\UFIB$. They also propose other methods to estimate $p_{\hat{Y} \mid A}$ which empirically performs better, but without guarantees; whereas $\UFIB$ has the nice property that $\UFIB$ is a consistent estimator of $\UFI$. This is because of the consistency of the Bayesian probability estimates and by applying the Continuous Mapping Theorem for $\max$ and $\min$ which are continuous.
Note that the empirical estimator (with $\alpha=0$) is also consistent, but has infinite bias. 

Nonetheless, $\UFIB$ has the drawback that for a low amount of samples and when the number of protected groups is high, it is almost determined deterministically by the parameter $\alpha$  and cannot be trusted. If $N_a\!=\!0$ for a protected group $a$, the estimated distribution is uniform on $\hat{Y}\mid A\!=\!a$ and this group does not affect the computation of the $\sup$ and $\inf$. Hence if the most discriminated group is among the undiscovered one, we risk making an important error on the estimation. When $N_a$ increases, we gain more information on the distribution of $\hat{Y} \mid A$. However, when $N_a$ is still low for all groups, the estimated distribution of the $\inf$ of $\hat{Y} \mid A\!=\!a$ is almost entirely determined by the prior parameter $\alpha$. 

\vspace{-3mm}

\subsection{Probabilistic Unfairness}

\vspace{-1mm}

When the number of protected subgroups grows arbitrarily large, it may be useless to try to guarantee fairness for every single one of them, regardless on how many people this truly affects. Should a decision maker sacrifice any potential predictive performance in order to guarantee fairness? It could be argued that an algorithm which discriminates $1$ person among a $1000$ can be described as fair to an extent. We may even be able to directly compensate the small amount of persons discriminated against if possible. Let us consider another example: if a company has clients on which it leverages machine learning predictions to make decisions, it would seem very limiting to guarantee fairness for clients among specific protected groups for whom we will almost never deal with. Nevertheless, if the underlying clients distribution changes, our decision making process should also reflect this change in terms of fairness.  
This motivates looking at unfairness probabilistically in $(\hat{Y},A,A')$. To do that we define the random unfairness $U\!=\!\ufi(\hat{Y},A,A')$ the random variable which corresponds to randomly choosing a prediction, and then independently selecting two protected groups according to $p_{A \mid \hat{Y}}$ to compare them. We now define our notion of probabilistic unfairness: 
\begin{definition} 
For $\epsilon \geq 0$ and $\delta \in [0,1]$, we say that classifier $h$ over distribution $\mathcal{D}$ is $(\epsilon,\delta)$-probably intersectionally fair if $\Pr(U > \epsilon) \leq \delta$.
\end{definition}
\vspace{-1mm}
It can be seen for some given $\epsilon$ as a statement on the expected size of the population that is not being discriminated  too much against. Probable intersectional fairness corresponds to searching for quantiles of $U$. We define the $\delta$-probabilistic unfairness as $\epsilon^*(\delta)\!=\!\min \{\epsilon \in \mathbb{R} \mid \Pr(U > \epsilon) \leq \delta \}$. It is the $(1-\delta)$-quantile of $U$. We also know by definition that any classifiers over any distributions is $(\UFI,0)$-probably intersectionally fair, as we have $U \leq \UFI$ with probability $1$. This shows that probabilistic fairness is a relaxed version of the hard intersectional unfairness as $\lim_{\delta \rightarrow 0}\epsilon^*(\delta)\!=\!\UFI$, and thus can be made arbitrarily close to intersectional fairness. In order to give more intuition on what this measure of fairness represents, we will briefly only for this paragraph consider discrimination of protected groups compared to the predictions distribution $p_{\hat{Y}}$ instead of between groups, meaning that we now measure $\vert \log(p_{\hat{Y} \mid A}/p_{\hat{Y}}) \vert $.
Suppose that a prediction model will be deployed over a population of $n$ individuals. Then if the classifier is $(\epsilon,\delta)$-probably intersectionally fair, this means that $\mathbb{E}_{A,\hat{Y}}[\sum_{i=1}^n \mathbbm{1}[u(\hat{Y}^{(i)},A^{(i)}) > \epsilon]]$ the expected number of people that faces a discrimination more than $\epsilon$ is less than $n\delta$. This allows us to measure and control the size of the population that may face a difference in treatment that would be deemed too high. It corresponds to the notion of fairness we were searching for. For more comparisons between these different notions, see Appendix \ref{app.comparison}.

As a remark, looking at $\mathbb{E}[U]$, it can serve as a lower bound of $\UFI$ because $\EUFI\!=\!\sum_{y,a,a'}p_{\hat{Y},A,A'}(y,a,a')\ufi(y,a,a')\leq \UFI$. This represents the average discrimination in a population between two protected groups. This is weaker than the notion presented above and is only mentioned in passing. 

Probabilistic fairness can be especially relevant in the context where $A$ are continuous sensitive attributes. Indeed, even for a very basic multivariate normal distribution on $A$, we will end up with $\UFI\!=\!\infty$ which is unhelpful. Yet by considering this notion of probabilistic fairness we end up with finite (hence comparable) measures of unfairness where the discriminated population size can be explicitly controlled; see Appendix \ref{app.continuous} for some examples. All in all, this notion of probabilistic unfairness, beyond its main interest of being a relaxed version of intersectional unfairness, could be in itself helpful for decision makers. 

\vspace{-3mm}

\section{Measures of Independence and Theoretical Bounds}\label{sec.bounds}

\vspace{-2mm}

We now focus on providing valid $(\epsilon,\delta)$ couples for probable intersectional fairness. First note that while the intersectional unfairness $\UFI$ is hard to estimate, it is much easier to estimate the marginal unfairness $\UFM$. The work done by \cite{Krich} in the different setting of weighted unfairness, however, shows through experiments that across multiple classifiers and data-sets, $\UFI$ and $\UFM$ can be uncorrelated, correlated, or even equal. Building on this observation, we would like to approach $\UFI$ using marginal quantities estimable for reasonably-sized data-sets.

\vspace{-3mm}

\subsection{Intersectional Unfairness with Independence}

\vspace{-1mm}

Since the intersectional unfairness takes into account the interactions between all the protected attributes $A_k$, one could guess that if the $A_k$ are mutually independent, this implies that $\UFI$ is close to $\UFM$. Our first result is not far from this intuition, but we also need to take into account the influence from the classifier $h$. Indeed, even if the protected attributes are independent, since the classifier makes predictions based on $X$ which may encode redundant information from some $A_k$, there can be interaction between those protected attributes through the classifier. See Appendix \ref{app.counter_example} for a counter example with the independence of the $A_k$ only but no clear relationship between marginal and intersectional fairness.

\begin{proposition}
 \label{thm.independent}
If the protected attributes $A_k$ are mutually independent and mutually independent conditionally on $\hat{Y}$, then \vspace{-3mm}
\begin{equation} \label{eq.indep}
\UFI=\sup_{y \in \mathcal{Y}} \sum_{k=1}^d \sup_{(a_k,a_k') \in \mathcal{A}_k^2} \ufi[k](y,a_k,a_k') \leq \sum_{k=1}^d \UFI[k].
\end{equation}
\end{proposition} \vspace{-3mm}
\begin{proof}[Sketch of proof]
The main idea is to decompose $p_A$ and $p_{A \mid \hat{Y}}$ as their products of marginals using the independence assumptions, and using the fact that the $\sup$ taken over a product of functions with independent variables is distributed over the product. The inequality is obtained because the $\sup$ of a sum is smaller than the sum of the $\sup$. See proof in Appendix \ref{app.thm_indep}.
\end{proof}
\vspace{-3mm}
This theorem gives us a first sense on how intersectional unfairness relates with marginal unfairness in some contexts. This shows us that if the independence conditions are fulfilled, then $\UFI$ becomes easy to estimate. What we provide here are conditions and a equation to derive a direct relationship between the intersectional unfairness and the marginal unfairness of each $A_k$. These are unfortunately too strong conditions to actually expect and are almost never randomly satisfied, but they help us give insight into the relationship between marginal and intersectional fairness. It also drives the analysis conducted in the next sub-section.
 We would like to relax the independence criteria while still using marginal information from the problem. 

\vspace{-3mm}

\subsection{Bounds on Probable Intersectional Fairness}

\vspace{-1mm}

In order to bound the probable intersectional unfairness and relate it with the strictly independent case, we want to use some measure of independence. We want to bound in probability the joint probability density $\Pr(A\!=\!a)$ with the product of its marginals $\prod_{k=1}^d \Pr(A_k\!=\!a_k)$. We will use one of the possible multivariable generalization of Mutual Information known as Total Correlation \cite{total_correlation_definition}:
\begin{equation}
C(A)\!=\!\mathbb{E}_{A} \!\Big[\!\log \! \Big(\!\frac{p_A(A)}{\prod_{k=1}^d p_{A_k}(A_k)}\!\Big)\!\Big]\!=\!\sum_{a \in \mathcal{A}}\! p_A(a) \! \log \! \Big( \frac{p_A(a)}{\prod_{k=1}^d p_{A_k}(a_k)}\Big) \!=\!\Big(\!\sum_{k=1}^d H(A_k)\!\Big) \! - \! H(A),
\end{equation}
where $H(A)$ is the Shannon Entropy of $A$. Similarly we define the conditional total correlation as $C(A\!\mid\!\hat{Y})\!=\!\mathbb{E}_{A,\hat{Y}}[\log(p_{A \mid \hat{Y}}(A \!\mid\! \hat{Y})/\prod_k p_{A_k \mid \hat{Y}} (A_k \!\mid\! \hat{Y}))]\!=\!(\sum_{k=1}^d H(A_k\!\mid\!\hat{Y})) - H(A\!\mid\!\hat{Y})$ where $H(A\!\mid\!\hat{Y})$ is the conditional entropy of $A$ given $\hat{Y}$. Note that both can also be written in terms of a KL or expectation in $\hat{Y}$ over conditional KL divergence, which means that $C(A)\!\geq\!0$ and $C(A\!\mid \!\hat{Y}) \!\geq \!0$. From these measures of independence, we intuitively define the following two random variables, $L\!=\!\log(p_A(A)/\prod_k p_{A_k}(A_k))$ and $L_y\!=\!\log(p_{A\mid \hat{Y}}(A\!\mid\! \hat{Y})/\prod_k p_{A_k \mid \hat{Y}}(A_k \!\mid\! \hat{Y}))$. By definition we have that $\mathbb{E}[L]\!=\!C(A)$ and $\mathbb{E}[L_y]\!=\!C(A \! \mid \! \hat{Y})$. We denote $\sigma$ and $\sigma_y$ the standard deviation of these two variables. We have the following property: \vspace{-3mm}

\begin{equation} \label{eq.sig_indep}
\indep_{k=1}^d A_k \Leftrightarrow C(A)=0 \Leftrightarrow\sigma=0 \quad \mbox{and} \quad \indep_{k=1}^d A_k \vert \hat{Y} \Leftrightarrow C(A \!\mid \!\hat{Y})=0 \Leftrightarrow \sigma_y=0. 
\end{equation}
The equivalence between independence and $C(A)\!=\!0$ comes from rewriting $C(A)$ as a $\KL$ and the fact that $\KL(P\Vert Q)\!=\!0$ if and only if $P\!=\!Q$ almost everywhere. For $C(A \!\mid \!\hat{Y})\!=\!\mathbb{E}_y[\KL(p_{A\mid \hat{Y}=y} \Vert \otimes p_{A_i\mid \hat{Y}=y})]$ we also use that the expectation of a positive random variable is $0$ if and only the variable is $0$ almost everywhere. When $\sigma\!=\!0$ then $L\!=\!c$ is a constant which means that $p_A\!=\!\prod_k p_{A_k} e^c$, and using that the probabilities must sum to $1$ we have $e^c\!=\!1$ hence $L\!=\!c\!=\!0$. The same arguments apply for $\sigma_y$. We denote $I(V,W)\!=\!H(V)-H(V\mid W)$ the mutual information between a variable $V$ and $W$. With these definitions, we can now derive the following theorem which bounds the probable intersectional fairness with independence measures and functions of marginal densities:

\begin{theorem} \label{thm.chebyshev}
For $\delta \in (0,1]$, any classifier $h$ over a distribution $\mathcal{D}$ is $(\epsilon_1,\delta)$ and $(\epsilon_2,\delta)$-probably intersectionally fair  with \vspace{-3mm}
\begin{align}\label{eq.epsilon1-bound}
&\epsilon_1\!=\!2\sqrt{2} \! \frac{\smax}{\sqrt{\delta}}\!+\!\sup_{y \in \mathcal{Y}} \Big\{ \! \sum_{k=1}^d\! \sup_{(a_k,a_k') \in \mathcal{A}_k^2}\! \ufi[k](y,a_k,a_k')\Big\} \\
&\epsilon_2\!=\!\sqrt{2}\frac{\smax}{\sqrt{\delta}}\!+\!\gamma\!+\!\sup_{y \in \mathcal{Y}} \Big\{\sum_{k=1}^d \!\log \!\Big( \!\frac{p_{\hat{Y}}^{1-1/d}(y)}{\inf_{a_k \in \mathcal{A}_k} p_{\hat{Y}\mid A_k}(y \! \mid \! a_k)} \!\Big) \!\Big\}
   \\
&\text{where} \quad \smax=(\sigma^{2/3}+\sigma_y^{2/3})^{3/2} \quad \text{and} \quad \gamma\!=\!C(A)-C(A\!\mid\!\hat{Y})\!=\!\big(\sum_{k=1}^d I(A_k,\hat{Y})\big)-I(A,\hat{Y}).
\end{align}
\end{theorem} \vspace{-3mm}

\begin{proof}[Sketch of proof]
We apply Chebyshev's inequality to $L$ and $L_y$ for some introduced parameters $\boldsymbol{\alpha}$ to bound the tails of these random variables, while making sure that overall the probability bounds stay larger than $1-\delta$. 
We can then compute inequalities on $p_A$ and $p_{A\mid \hat{Y}}$, and take the $\inf$ for $a$ and $\sup$ for $\boldsymbol{\alpha}$. This leads to a constrained minimization problem that can be solved, which yields $\smax$. The full proof is in Appendix~\ref{app.chebyshev}. For $\epsilon_2$ we additionally use that $p_{\hat{Y} \mid A} \leq 1$ as $\mathcal{Y}$ is discrete. \vspace{-3mm}
\end{proof}

We observe that both $\epsilon_1$ and $\epsilon_2$ are composed of one term in $\smax$ related with the $\delta$-confidence, and a quantity with marginal information. Aditionnaly $\epsilon_2$ also includes a term in $\gamma$ that corresponds to some form of mutual information correction. We can control the confidence in this bound with the parameter $\delta$. Because $\smax\!=\!0$ if and only if $\sigma\!=\!\sigma_y\!=\!0$ and combined with \eqref{eq.sig_indep} we can see that $\smax$ somewhat measures how far we are from the conditions of Proposition \ref{thm.independent}. With $\epsilon_1$ we see that when $s^*$ goes to zero, we recover exactly the conditions of Proposition \ref{thm.independent}.

In order to prove Theorem \ref{thm.chebyshev}, we used Chebyshev's inequality. We can derive a similar proof for other concentration inequalities, specifically with Chernoff bounds through the estimation of the moment generating function, which often leads to tighter bounds. However this leads to harder quantities to estimate in addition to having to solve a non-convex optimization problem, see Appendix \ref{app.chernoff}.

To conclude this section we provide additional intuition on the relationship between marginal and intersectional fairness. We can then derive the following corollary from the proof of the above Theorem:
\begin{corollary} \label{cor.fraction}
Denoting $(\Omega,\mathcal{T},\Pr)$ the probability space on which $(\hat{Y},A,A')$ is defined, there exists an event $F$ so that for $f(y,a)=\prod_{k=1}^d (p_{\hat{Y} \mid A_k}(y \mid a_k)/p_{\hat{Y}}(y))$ we have for $U$ the random unfairness:
\begin{equation}
-\frac{2 \sqrt{2}\smax}{\sqrt{\delta}} + \sup_{\omega \in F}\left \vert \log\left( \frac{f(\hat{Y},A')}{f(\hat{Y},A)}\right)\!(\omega)  \right  \vert \leq \sup_{\omega \in F} U(\omega) \leq \frac{2 \sqrt{2}\smax}{\sqrt{\delta}} + \sup_{\omega \in F}\left \vert \log\left( \frac{f(\hat{Y},A')}{f(\hat{Y},A)}\right)\!(\omega)  \right  \vert,
\end{equation}
with $\Pr(F)\geq 1- \delta$. 
\end{corollary}

The proof can be found in Appendix \ref{app.chebyshev}. This means that there is a fraction of the relevant pairs population of size bigger than $1-\delta$, for which we can give an interval for the maximum random unfairness over this fraction $F$. This interval is centered and reduce around a unique quantity as $\smax$ goes to $0$, with $\smax$ and $\delta$ determining the length of this interval. When $\delta$ goes to $0$, we have $\Pr(F) \rightarrow 1$ hence  $\sup_{\omega \in F} U(\omega) \rightarrow \UFI$ because we are dealing with finite random variables. Notice also that when both $\delta$ and $\smax$ go to $0$, we recover Proposition \ref{thm.independent} as $\sup_{\omega \in F} \vert \log( f(\hat{Y},A')/f(\hat{Y},A))(\omega)  \vert$ goes toward the quantity derived in this Proposition when $\delta$ goes to $0$.

\vspace{-3mm}
\subsection{Estimation of the measures of independence}
\vspace{-1mm}

Theorem \ref{thm.chebyshev} trades the precise estimation of $\UFI$ with an upper bound, but with much easier quantities to estimate. More specifically, as they are information measures, we can leverage the extensive literature on statistical estimators and entropy estimation. We can intuitively see that the estimation of $\smax$ and $\gamma$ will be easier to handle because even the estimation with the empirical distribution $\hat{p}_{A,\hat{Y}}$ is always well defined, and is a \ac{MLE} as continuous functions of \ac{MLE}. They are well defined because $\smax$ and $\gamma$ are functions of entropies and of the quantities $Q(P)\!=\!\sum_i p_i \log(p_i)^2$ for a probability distribution $P$, which is finite event for $p_i\!=\!0$ because $x \mapsto x \log(x)$ and $x\mapsto x \log(x)^2$ are continuous at $0$.  Contrarily to $\UFIB$ we do not have to use any prior to obtain a well defined estimator. In addition, using the delta method on the sum of entropies, for which the \ac{MLE} is asymptotically normal (See \cite{entropy_paninski} 3.1), shows that $\hat{\gamma}$ is asymptotically normal.
 
For more information on the estimation of entropy, mutual information or total correlation we defer to \cite{entropy_paninski,entropy_nsb,entropy_bayesian,total_correlation_estimation,minimax_kl} to name but a few. Moreover even with the very simple \ac{MLE}, we can obtain $L_2$ error upper-bounds for $H(P)$ in $\mathcal{O}(\log(\vert P\vert)^2/n)$ where $\vert P\vert$ is the number of outcomes for a discrete distribution $P$ \cite{discrete_functionals}. This bound depends only on the number of outcomes (supposed known), and not the actual distribution. Using the same tools, we derive a rough error bound for $Q(P)$: 
\begin{proposition} \label{prop.err_L2} 
\begin{equation} 
\mathbb{E}[(Q(P)-Q(\hat{P}))^2]=\mathcal{O}\Big(\frac{\log^4(n)}{n}\Big).
\end{equation}
\end{proposition} \vspace{-3mm}
\begin{proof}[Sketch of proof]
We apply the methods described in \cite{discrete_functionals} that bounds the bias using approximation theory for Bernstein polynomials and bounds the variance using the Efron-Stein inequality. See proof in Appendix \ref{app.plogp2}.
\end{proof} 
\vspace{-3mm} More efficient estimators can be created using methods of \cite{minimax_functional}, nevertheless the main interest of this proposition is to show that these quantities have an error rate depending on the number of samples $n$, and not the number of samples per group $N_{a}$, which is much better.

Beyond the practical use of these inequalities and approximations, these theorems also show one crucial idea: we can relate intersectional and marginal unfairness with the help of information on the independence of the protected attributes.

\vspace{-3mm}
\section{Refined approximations and inequalities}
\vspace{-2mm}

In the previous section, we have derived conditions for marginal unfairness to directly relate to $\UFI$, and bounds on probable intersectional fairness. We now would like to propose an approximation of $\UFI$ using similar ideas. Looking at \eqref{eq.epsilon1-bound}, \eqref{eq.indep}, and indirectly through Corollary \ref{cor.fraction} it seems natural to propose as one possible approximation of $\UFI$ the following quantity:
 \begin{equation}
 \UFIind=\sup_{y \in \mathcal{Y}} \sum_{k=1}^d \sup_{(a_k,a_k') \in \mathcal{A}_k^2} \ufi[k](y,a_k,a_k').
 \end{equation}
For the rest of the article, we thus now only focus on $\smax$ and $\UFIind$. Compared with $\UFIB$ the estimator with Bayesian prior, it does not depend on a prior parameter, and is usually well defined as we only need $N_{a_i,y}$ the number of samples per $a_i$ and $y$ to be strictly positive instead of all $N_{a,y}$. However this estimator of $\UFI$ is not consistent. We will show that the previous bounds can be improved and that we can make our estimator consistent by gradually grouping together the protected attributes as the number of samples increases. 
 
\vspace{-3mm}
\subsection{Grouping protected Attributes together}
\vspace{-1mm}

Until now, we have always decomposed the protected attributes $A$ on their marginals $A_i$. However it may be that we have more than just marginal information available. Take the example of 4 protected attributes $A\!=\!(A_1,A_2,A_3,A_4)$. For a set $t \subseteq \{1,2,3,4\}$, we define $A_t\!=\!(A_k)_{k \in t}$. We may not have enough data to compute the full intersectional unfairness, but it may be possible to compute it for the grouped protected attributes $A_{\{1,2\}}\!=\!(A_1,A_2)$ and $A_{\{3,4\}}$. We can use the same decomposition as we did before on the new marginals attributes (which corresponds to flattening $A_1$ and $A_2$ together) with support $\mathcal{A}_{\{1,2\}}\!=\!\mathcal{A}_1 \times \mathcal{A}_2$ and $\mathcal{A}_{\{3,4\}}\!=\!\mathcal{A}_3 \times \mathcal{A}_4$. 

More generally, let $q$ be a partition of $\{1,...,d\}\!=\![d]$. For a partition $q$, we denote $A^{(q)}\!=\!(A_t)_{t \in q}$. This is only a different way to group together the marginal attributes, and is the same as $A$. Whenever quantities are changed according to some partition $q$, it will be indicated with $(q)$. For each of the new marginal attributes defined by a set $t$ of $q$, the new marginal unfairness $\UFI[t]$ corresponds to the intersectional unfairness of the $(A_k)_{k \in t}$. If the $A_t$ are independent, and independent conditionally on $\hat{Y}$, we can apply Proposition~\ref{thm.independent} and obtain directly $\UFI$ through the newly defined marginals. If we relax the independence conditions, the same arguments of the previous section still apply, and we can look at the bounds and approximations defined by these new marginal densities. 
We denote the new approximation with partition $q$ by $\UFIind[(q)]\!=\!\sup_{y \in \mathcal{Y}} \sum_{t \in q} \sup_{(a_t,a_t') \in \mathcal{A}_t^2} \ufi[t](y,a_t,a_t')$ 
where we are using the new marginals defined by $q$. If we use the partition $q$ of singletons then $\UFIind[(q)]\!=\!\UFIind$, and if we use the trivial partition (the whole set) then $\UFIind[(q)]\!=\!\UFI$. The constraints of independence for these new marginals should be more feasible than the original marginals, hence it is possible that the $A_t$ fulfill the independence conditions, even if the $A_k$ do not (the trivial partition is such an example). If we have enough data to compute the marginal densities derived from $q$ and the $A_t$ fulfill the independence conditions, we can then compute $\UFI$ through the partition $q$. Of course most of the time the independence conditions are not satisfied satisfied for a partition $q$. Nonetheless because $\smax$ measures how far we are from the independence conditions, we can more carefully select a partition among those for which we can compute the new marginal densities.

\vspace{-3mm}
\subsection{Efficient Partition Selection}
\vspace{-1mm}

Let $\mathcal{Q}$ be the set of all \emph{feasible} partitions $q$, that is $\mathcal{Q}\!=\!\{q \in \mathcal{P}([d])\! \mid \! \forall t \in q, \forall (a_t,y) \in \mathcal{A}_t\times \mathcal{Y}, N_{a_t,y}>0 \}$ with $\mathcal{P}([d])$ the set of all partitions of $[d]$. This set represents the set of partitions for which we can compute the newly defined marginals without having to use a prior parameter. Note that $\mathcal{Q}$ is a random set that converges to $\mathcal{P}([d])$ almost surely as the number of samples $n$ increases. If $q' \in \mathcal{Q}$, then any partitions $q$ finer than $q'$ (meaning that any element of $q$ is a subset of an element of $q'$) is in $\mathcal{Q}$ as well. We will say that $q$ can be merged further if there exists a partition $q' \in \mathcal{Q}$ so that $q$ is finer than $q'$. Note that the choice of a partition $q$ does not change the value of $\UFI$ but only that of $\UFIind[(q)]$. We therefore want to find a good feasible partition $q$ in $\mathcal{Q}$ so that we can expect heuristically $\vert \UFIind[(q)]-\UFI \vert$ to be the lowest among the partitions. There are two criteria that should help us decide which partition $q \in \mathcal{Q}$ to choose from.

 \begin{wrapfigure}{L}{0.6\textwidth}
\begin{minipage}{0.6\textwidth}
\begin{algorithm}[H]
\caption{Greedy Partition Finder}\label{alg:greedy}
\begin{algorithmic}
\State \textbf{input:} Protected attributes data and occurrences of $\hat{Y}$ 
\State \textbf{require:} The partition of singletons is feasible
\State $q^*$ $\leftarrow$ the partition of singletons
\Repeat 
    \State $\mathcal{M}\!=\!\{ \! \{t_1 \!\cup t_2\}\cup q^* \setminus\!(\{t_1\}\!\cup\! \{t_2\}), (t_1,t_2) \! \in \!q^{*2} \!,t_1 \! \neq \! t_2\}$
    \State $s^*_{\text{min}} \leftarrow +\infty$
    \For{$q$ in $\mathcal{M}$}
        \If{$q$ is feasible and $\smax(q)<s^*_{\text{min}}$}
            \State $(s^*_{\text{min}},q^*) \leftarrow (\smax(q),q)$
        \EndIf
    \EndFor  
\Until{$\mathcal{M}=\emptyset$ or $s^*_{\text{min}} = \infty$} (Nothing possible to merge)
\State \textbf{return:} $q^*$
\end{algorithmic}
\end{algorithm}
\end{minipage}
\end{wrapfigure} 

If a partition $q'$ is coarser than $q$ (which means that $q$ is finer than $q'$), then reasonably the approximation is better with $q'$ than $q$. The reasoning is that by taking coarser partitions, we are taking more interactions between the protected attributes into account. For example the coarsest partition which is the whole set gives us the intersectional unfairness as mentioned earlier. However because the `finer-than' relationship is only a partial order, we are not able to choose between any two sets. Because Theorem $\ref{thm.chebyshev}$ seems to hint that there is a relationship between the error $\vert \UFIind[(q)]-\UFI \vert$, and the distance to the independence conditions $\smax$, the second criterion will be to select the partitions $q$ with the smallest $\smax(q)$ defined as $\smax$ but taking the marginals in $q$. 
These two criteria are closely linked. Selecting coarser partitions does tend to yield partitions with smaller $\smax$, but not always. We give some details on relationship between $\smax(q)$ of a partition $q$ compared to a coarser one in Appendix \ref{app.partitions}. More crucially, finding a good partition with a small $\smax(q)$ will also improve our inequalities as they are a function of $\smax(q)$ which decreases on average as shown in Figure \ref{sec.experiments} as the number of sample grows (and as the partitions get coarser). 

In principle, finding the best partition according to our criteria requires enumerating all feasible partitions which is computationally intractable. Instead we propose a greedy heuristic that we describe in Algorithm \ref{alg:greedy}. We start from the finest partition (the partition of singletons), look at all the feasible partitions (with enough data) that can be obtained from merging two elements of the current partition, select the one with the smallest $\smax$, and repeat until there are no coarser partitions with enough data.
Note that when we want to verify that there is enough data available, we may need to do it multiple times for the same subset of protected attributes. This is an expensive call so it is more efficient to do memoization and remember if there is enough data available for a given subset once encountered, which we do using a hash table to reduce the lookup time. We denote this partition $q^*$. We have the following property with the proof in Appendix \ref{app.consistency}: 
\begin{proposition}
\label{thm.consistent}
The estimator $\UFIind[(q^*)]$ is a consistent estimator of $\UFI$.
\end{proposition}  
This proposition shows that $\UFIind[(q^*)]$ is relevant in estimating $\UFI$, while not needing to use a Bayesian prior with parameters that may overwhelmingly affect the estimation. Note that instead of using $\mathcal{Q}$ which ensures that $N_{a_t,y}>0$, we can instead use $\mathcal{Q}_{\tau}$ for $\tau \in \mathbb{N}$ which ensures that for any $q \in \mathcal{Q}_{\tau}$, $N_{a_t,y}>\tau$ for $t \in q$. 
\vspace{-3mm}
\section{Experiments} \label{sec.experiments}
\vspace{-2mm}
In this section, we present experimental results that show how our inequalities and approximations perform on real and synthetic data-sets, and compare their estimation error rates as the number of samples grow. All the code used in our experiments can be found in the supplementary material or at  \href{https://github.com/mathieu-molina/BoundApproxInterMargFairness}{\textcolor{blue}{https://github.com/mathieu-molina/BoundApproxInterMargFairness}}. 
\vspace{-3mm}
\subsection{Data-sets and processing}
\vspace{-1mm}

In order to compare how well $\UFIB$ and $\UFIind$ perform as estimators on data on datasets with a high number of protected attributes, we need to compute $\UFI$ which is as discussed above inherently difficult. We will always measure the unfairness with respects to the empirical distribution of the dataset. For this empirical distribution to yield a well defined fairness measure, we need that $N_{a,y}>0$ for all $a$ and $y$. 

This means that if we want to take into account a high number of sensitive attributes, we have to pick a very large dataset.

We used  \href{https://archive.ics.uci.edu/ml/datasets/US+Census+Data+(1990)}{\textcolor{blue}{US Census data from 1990}} \cite{UCIarchive} which contains $n\!=\!2,458,285$ samples, and for which we identified many potential protected attributes. We then train a Random Forest binary classifier on a poverty binary label, where we weight the labels differently so as to obtain about the same number of predictions for each outcome. However we still do not have $N_{a,y}>0$ on the whole dataset. To alleviate this issue, we will consider subsets of the protected attributes for which this is true, and we will measure fairness with respect to these subsets.
We obtain about $100$ different subsets with $d=8$ protected attributes, that we denote as $D_i$ which is the original dataset where we only kept the $i$-th subset of protected attributes and the predictions $\hat{Y}$. Each of these subsets yield different values of $\UFI$ and $\smax$. We pick $12$ (for computational reasons) different $D_i$ with various values of $\UFI$ and $\smax$. Some examples of the final protected attributes include sex, not speaking English at home, being overweight, being Hispanic, and others. We will always take $\delta\!=\!0.1$ when relevant. 

We also conduct experiments on synthetic data. We generate $(A,\hat{Y})$ probability distributions from a Dirichlet distribution, thus we can directly compute $\UFI$ without dealing with a very large dataset. We take $d\!=\!10$. This synthetic data is one of the worst case for the approximation of $\UFI$ with $\UFIind$, as the marginal distributions are a sum of $2^{d-1}$ \ac{iid} random variables that all converges to $1/2$ as $d$ grows. We therefore will not plot $\UFIind$ for the synthetic data (it is close to $0$). Nonetheless, this synthetic data remains useful in order to compare the error rates between $\UFIB$ and $\hat{s}^*$ and their respective true value. We denote by $P_i$ a generated probability distribution. We generate $12$ of them.
\vspace{-3mm}
\subsection{Experiments Results} 
\vspace{-1mm}
We first want to compare the convergence rate of $\UFIB$, $\smax$ and $\UFIind$ to their asymptotic value. To do that, and because they can take different values, we compute for each estimator $\hat{T}_n$ that converges in probability to $T$ the relative expected $L_2$ error rate $L_2^r(\hat{T}_n)\!=\!\mathbb{E}[(\hat{T}_n-T)^2/T^2]$.  
We fix a number of available samples $n$ from $100$ to $2,000$, and we sample without replacement from the datasets.  From these available samples, we compute all our estimators. We denote by $\hat{u}_{I}$ the estimator of $\UFIind$ computed with the empirical marginal densities for $n$ samples. In order to compute $L_2^r$, for each subset and each sample size $n$, we sample from $D_i$ and $P_i$ $20$ times for each fixed number of samples $n$. 

We see in Figure \ref{fig.comp_conv} on the left-most plot, that $\hat{u}_I$ is reasonably close to $\UFI$ on average. Still the gap between $\epsilon^*$ and $\epsilon_1$ is quite big. Other bounds such as $\epsilon_2$ or with other concentration inequalities are generally a bit more efficient but still loose, nevertheless we focus here on comparing the error rates between $\smax$ and $\UFIB$, and on how $\UFIind$ performs. The other two plots look at the $L_2^r$ for the various estimators, with the middle one being with the real datasets $D_i$, and the right on the synthetic datasets $P_i$. We can see that $\hat{s}^*$ and $\hat{u}_{I}$ converges much faster than $\UFIB$. Moreover, it seems that the difference in error rate will only grow bigger as $d$ increases, as there is a bigger gap for $P_i$. We can also see that $\UFIB$ is unreliable, because the error rate varies a lot depending on $\alpha$, and can even increase. This is because the parameter $\alpha$ dominates the computation of $\UFIB$ as discussed earlier.

\begin{figure}[ht] 
\begin{center}
\centerline{\includegraphics[width=1.1\columnwidth]{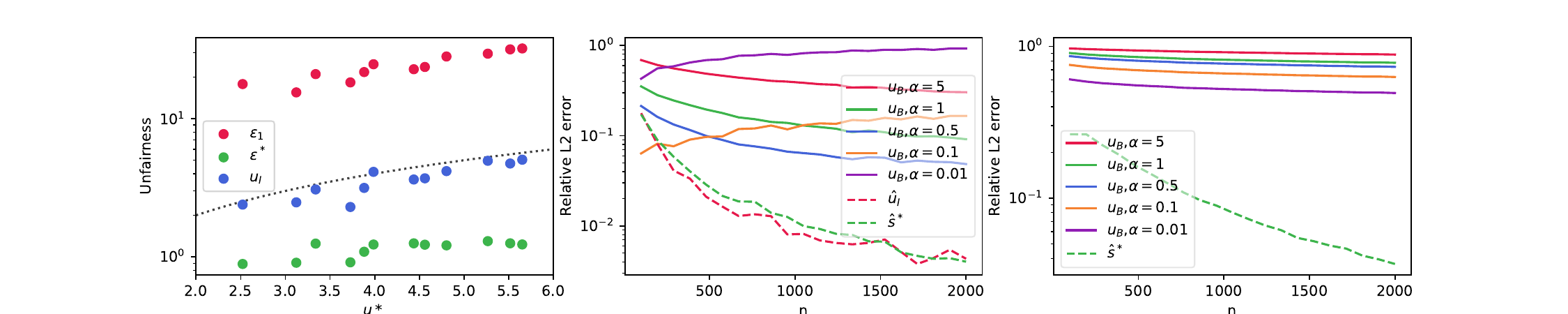}}
\caption{On the left-most plot, each point represents one real dataset $D_i$, and we compare $\epsilon_1$, $\epsilon^*$, and $\UFIind$ with $\UFI$. The dotted line corresponds to the equation $x=y$ for reference. The middle plot describes the average over the $D_i$ of $L_2^r$ as $n$ increases for $\hat{u}_{I}$, $\hat{s}^*$, and $\tilde{u}^*$. $\UFIB$ is computed for multiple values of $\alpha$. The right-most plot is similar, but uses the synthetic datasets $P_i$.}
\label{fig.comp_conv}
\vspace{-3mm}
\end{center}
\end{figure}
\vspace{-3mm}
We now conduct similar experiments, but this time using partitions. We can see in the middle plot of Figure $\ref{fig.exp_part}$ that $\hat{u}_{I}^{(q^*)}$ performs better. The choice of $\tau$ the count threshold for grouping always gives reasonable approximations, with $\tau\!=\!1$ being close to $\UFIB$, and $\tau$ big makes it close to $\UFIind$. Most importantly, the apparent good error rate of $\UFIB$ is merely an artifact of the current range of $\UFI$ being above the starting values of $\UFIB$ for these $\alpha$. 
It is clear that $\UFIB$ is unreliable by looking at the left plot in Figure $\ref{fig.exp_part}$: the estimation with $\UFIB$ at $n=2000$ for different values of $\alpha$ varies very little when $\UFI$ varies (it is almost not a function of $\UFI$). This means that $\UFIB$ depends very little on the data for low amount of samples. Even if it is not perfect, $\UFIind[(q^*)]$ still has better performance and is more coherent. We note that the approximation performs well comparatively only when $d$ is high, and considering more sensitive attributes should make an even bigger difference. These results combined with Proposition \ref{thm.consistent} show that $\UFIind[(q^*)]$ is a relevant estimator of $\UFI$ with scarce data and high number of protected attributes. Concerning $\smax(q^*)$ the right-most plot shows that while it is not completely monotone, $\smax(q^*)$ does decreases on average when using partitions as the number of sample increases. The upper bound will become tighter as $n$ grows, which will make bigger groupings of protected attributes possible. 
\vspace{-3mm}
\begin{figure}[ht] 
\begin{center}
\centerline{\includegraphics[width=1.1\columnwidth]{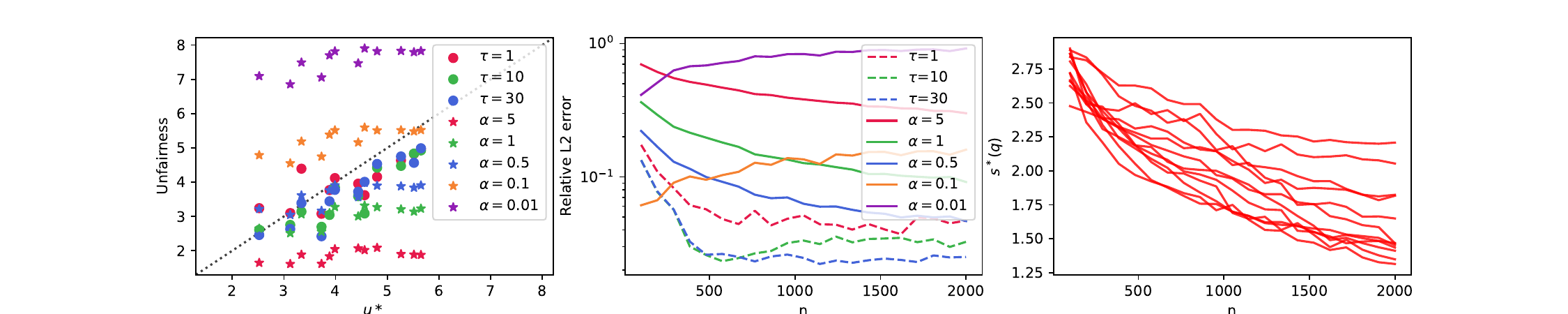}}
\caption{Each point of the same color and shape represent the estimation for one dataset $D_i$. The estimations are computed for $n=2000$. The middle plot is the same as above, but with $\hat{u}_{I}^{(q^*)}$ this time. The rightmost plot is the average evolution of $s^*(q^*)$ for $\tau=10$ as $n$ increases.}
\label{fig.exp_part}
\end{center}
\end{figure}
\vspace{-1mm}

\vspace{-8mm}
\section{Discussion}\label{sec:discussion}
\vspace{-2mm}
In this work, we presented new methods to approximate and to bound (in high probability) a strong intersectional unfairness measure, based on statistical information computable from a reasonable dataset. Our results highlight the key role of independence of the protected attributes conditionally to the classifier, and propose to approach it via a smart grouping of some attributes---which our theoretical bound allows us to compute via an efficient heuristic. 

Our experiments show that the approximations proposed here perform reasonably well for data-sets with a high number of protected attributes, but that our bounds are not very effective. However their main interest is that it gives insight into the link between marginal and intersectional fairness, which was the main goal of this work. It also helps us derive the proposed approximation. We expect that more effective bounds could be derived for our notion of probabilistic fairness, for instance by making additional assumptions on the distribution, but presumably without an explicit dependence on independence measures and marginal densities, making the link between marginal and intersectional fairness harder to see. 

In order to train fair models using the proposed approximations or bounds of this paper, we can use soft counts to compute the empirical densities (based on the classifier score for instance) as suggested in \cite{Fintersectional}. This makes the approximations and bounds differentiable, and ensure that we can apply gradient based methods so as to solve a constrained or penalized optimization problem using these quantities.

We hope that our approach will enable the development of improved bounds, raise interest in the proposed notion of probabilistic unfairness which we think is crucial to the development of fair algorithms, as well as the use of our approximations to penalize classifiers in order to train intersectionally fair classifiers. 

\vspace{-3mm}
\section*{Acknowledgments}
\vspace{-1mm}

This work has been partially supported by MIAI @ Grenoble Alpes (ANR-19-
P3IA-0003), by the French National Research Agency (ANR) through grant
ANR-20-CE23-0007, and by TAILOR (a project funded by EU Horizon 2020
research and innovation programme under GA No 952215). The authors are hosted at the CREST lab (CNRS, GENES, Ecole Polytechnique, Institut Polytechnique de Paris).

\newpage

\bibliographystyle{plain}
\bibliography{references.bib}

\newpage
\section*{Checklist}
\begin{enumerate}

\item For all authors...
\begin{enumerate}
  \item Do the main claims made in the abstract and introduction accurately reflect the paper's contributions and scope?
    \answerYes{}
  \item Did you describe the limitations of your work?
    \answerYes{}
  \item Did you discuss any potential negative societal impacts of your work?
    \answerYes{This work is specifically related to Fairness, and as such we highlight in the Introduction existing works which already discuss some of the societal issues that intersectional fairness raises.}
  \item Have you read the ethics review guidelines and ensured that your paper conforms to them?
    \answerYes{}
\end{enumerate}

\item If you are including theoretical results...
\begin{enumerate}
  \item Did you state the full set of assumptions of all theoretical results?
    \answerYes{}
        \item Did you include complete proofs of all theoretical results?
    \answerYes{}
\end{enumerate}

\item If you ran experiments...
\begin{enumerate}
  \item Did you include the code, data, and instructions needed to reproduce the main experimental results (either in the supplemental material or as a URL)?
    \answerYes{}
  \item Did you specify all the training details (e.g., data splits, hyperparameters, how they were chosen)?
    \answerYes{We specified relevant data processing details, the actual supervised learning model used not being the main interest. More details are given in the code.}
        \item Did you report error bars (e.g., with respect to the random seed after running experiments multiple times)?
    \answerYes{}
        \item Did you include the total amount of compute and the type of resources used (e.g., type of GPUs, internal cluster, or cloud provider)?
    \answerYes{}
\end{enumerate}

\item If you are using existing assets (e.g., code, data, models) or curating/releasing new assets...
\begin{enumerate}
  \item If your work uses existing assets, did you cite the creators?
    \answerYes{}
  \item Did you mention the license of the assets?
    \answerNA{}
  \item Did you include any new assets either in the supplemental material or as a URL?
    \answerNA{}
  \item Did you discuss whether and how consent was obtained from people whose data you're using/curating?
    \answerYes{}
  \item Did you discuss whether the data you are using/curating contains personally identifiable information or offensive content?
    \answerYes{}
\end{enumerate}

\item If you used crowdsourcing or conducted research with human subjects...
\begin{enumerate}
  \item Did you include the full text of instructions given to participants and screenshots, if applicable?
    \answerNA{}
  \item Did you describe any potential participant risks, with links to Institutional Review Board (IRB) approvals, if applicable?
    \answerNA{}
  \item Did you include the estimated hourly wage paid to participants and the total amount spent on participant compensation?
    \answerNA{}
\end{enumerate}

\end{enumerate}

\newpage

\appendix

\section{Additional elements on Measures of Fairness} 

\subsection{Generalization to other measures of fairness} \label{app.gen_other_fairness}

Throughout the whole paper, we used a specific measure of fairness for simplicity. Nevertheless, the same arguments apply to a broader set of fairness measures, by modifying $\UFI$.

To define $\UFI$, we decided to take the $\log$ of a ratio. We note that when taking the $\sup$ over all possibles $a$ and $a'$ in $\mathcal{A}$, $\sup_{\mathcal{A}^2} \ufi(1,\cdot,\cdot)\leq \epsilon$ is equivalent to the definition in \cite{Fintersectional}, that is to say:
\begin{equation}
\forall (a,a') \in \mathcal{A}^2, \quad e^{-\epsilon} \leq \frac{\Pr(\hat{Y}=1 \mid A=a)}{\Pr(\hat{Y}=1 \mid A=a')} \leq e^{\epsilon}.
\end{equation}
However if we want to define a measure of unfairness between two protected groups, it is reasonable for it to be symmetric in the groups considered. We make it so by applying $\log$ and an absolute value function to the above middle quantity. Other distances and pseudo distances can also be chosen, such as $\vert \Pr(\hat{Y}=1 \mid A=a) - \Pr(\hat{Y}=1 \mid A=a') \vert $. It is also symmetric, but may be less useful in comparing models with many protected attributes that are not designed to be fair, as this measure will be close to $1$ most of the time, making the comparison less precise between two different models.\\

We can also modify $U$ and the definition of probabilistic fairness accordingly to obtain other desirable measures of unfairness. Say that we are only interest in the unfairness related to the outcome $\hat{Y}=1$. Taking $U'=u(1,A,A')$ and $\Pr(U' > \epsilon \mid \hat{Y}=1)\leq \delta$ the new definition of probabilistic fairness in this case, we can derive similar propositions and theorems as done in this paper. We only need to take the expectation and variance with respect to $\Pr( \cdot \! \mid \! \hat{Y}=1)$ for $L$ and $L_y$. This yields statistical quantities which are harder to interpret ($\sum_{a \in \mathcal{A}} \Pr(A=a \mid \hat{Y}=1) \log(\Pr(A=a) / \prod_{k=1}^d \Pr(A_k=a_k)$) but that should remain easy to estimate as they are always well defined because, considering the empirical distribution $\hat{p}$ we have $\hat{p}_A(A=a)=0 \implies \hat{p}_A(A=a\mid \hat{Y}=1 )=0$.
The changed definitions would be the following if we are only interested in the outcome $y \in \mathcal{Y}$:
\begin{align} \label{eq.new_defs}
&\UFI= \sup_{(a,a')\in \mathcal{A}^2} \ufi(a,a'), \quad \text{and} \quad \UFI[k]=\sup_{(a_k,a_k')\in \mathcal{A}_k^2} \ufi[k]( a_k,a_k') \\
\text{with} \ &\ufi ( a,a') \!=\! \Big\vert \log \! \Big(\frac{\Pr(\hat{Y}\!=\!y \! \mid \! A\!=\!a)}{\Pr(\hat{Y}\!=\!y \! \mid \! A'\!=\!a')} \Big) \! \Big\vert ,\ \ufi[k] (a_k,a_k')\!=\!\Big\vert \log \! \Big(\frac{\Pr(\hat{Y}\!=\!y \! \mid \! A_k\!=\!a_k)}{\Pr(\hat{Y}\!=\!y \! \mid \! A'_k\!=\!a_k')}  \Big)\! \Big \vert,\\
&\gamma=\mathbb{E}[L-L_y \mid \hat{Y}=y], \quad \sigma=\sqrt{\Var(L \mid \hat{Y}=y)}, \quad \sigma_y=\sqrt{\Var(L_y \mid \hat{Y}=1)}. 
\end{align}
Notably, we obtain a much nicer variant of Proposition \ref{thm.independent}, with $\UFI=\sum_{k=1}^d \UFI[k]$.\\ 

Similarly we can also change only the underlying probability distribution. We can replace the underlying probability $\Pr$ by $\Pr_{Y=1}$. Using this new probability measure we see that $\UFI$ is a relaxed version of Equality of Opportunity for a binary predictor defined in \cite{equalopportunity} by 
\begin{equation}
\Pr(\hat{Y}=1 \mid A=a, Y=1)=\Pr(\hat{Y}=1 \mid Y=1).
\end{equation}
Indeed, $\UFI=0$ is now equivalent with Equality of Opportunity. Practically, changing the underlying probability does not make much difference as showed in \cite{Kgerrymandering} because this amounts to measuring unfairness on the part of the dataset for which $Y=1$. \\

We compared the treatment faced by groups between them, such as looking at the discrimination between men and women. Another possibility is to measure the difference in treatment faced by a group compared to a reference value. This reference value is most of the time taken to be the population average of the decision criterion $\mathbb{E}_A[p_{\hat{Y}\mid A}(y \mid A)]=p_{\hat{Y}}(y)$. Therefore instead of evaluating $p_{\hat{Y} \mid A}(y \mid a)/p_{\hat{Y} \mid A}(y \mid a')$ we evaluate $p_{\hat{Y} \mid A}(y \mid a)/p_{\hat{Y}}(y)$. \\

Finally we can change over which treatment criterion we want to evaluate differences. In this paper we decided to look at the variable $\hat{Y}$. We can similarly define our fairness measure with $Y$. We can actually use any $(X,A,Y)$-measurable random variable $Z$ instead of $\hat{Y}$. For instance $\vert \hat{Y} -Y \vert$, which tells us whether or not the prediction is correct for binary classification, can be a good candidate. 

Combining all of the above comments, we can consider a wider array of fairness metrics for which variations of the techniques and theorems described in this paper apply.

\subsection{Some variants of Theorem \ref{thm.chebyshev} for modified fairness measures}

\paragraph{Intersectional Fairness in terms of absolute difference:} We consider the following definition of unfairness:
\begin{align}
&\UFI=\sup_{y \in \mathcal{Y}} \sup_{(a,a')\in \mathcal{A}^2} \ufi(y , a,a') \\
\text{with} \quad &\ufi (y , a,a') = \Big\vert \log \Pr(\hat{Y}=y  \mid  A=a)-\Pr(\hat{Y}=y  \mid  A'=a')  \Big\vert.
\end{align}
The new version of Theorem \ref{thm.chebyshev} is 
\begin{theorem*}
For $\delta \in (0,1]$, any classifier $h$ over a distribution $\mathcal{D}$ is $(\epsilon_1,\delta)$-probably intersectionally fair  with
\begin{equation}
\begin{split}
\epsilon_1=e^{-\gamma} \sup_{\mathcal{Y}} p_{\hat{Y}}^{1-d}  \bigg(e^{\frac{\sqrt{2}s^*}{\sqrt{\delta}}}\prod_{k=1}^d \sup_{\mathcal{A}_k} p_{\hat{Y} \mid A_k} 
- e^{-\frac{\sqrt{2}s^*}{\sqrt{\delta}}} \prod_{k=1}^d \inf_{\mathcal{A}_k} p_{\hat{Y} \mid A_k} \bigg).
\end{split}
\end{equation}
\end{theorem*}

\paragraph{Intersectional Fairness when comparing to the population average:} We consider the following definition of unfairness:
\begin{align}
&\UFI=\sup_{y \in \mathcal{Y}} \sup_{a\in \mathcal{A}} \ufi(y , a) \\
\text{with} \quad &\ufi (y , a) = \Big\vert \log\Big( \frac{\Pr(\hat{Y}=y  \mid  A=a)}{\Pr(\hat{Y}=y)} \Big) \Big\vert.
\end{align}
The new version of Theorem \ref{thm.chebyshev} is 
\begin{theorem*}
For $\delta \in (0,1]$, any classifier $h$ over a distribution $\mathcal{D}$ is $(\epsilon_1,\delta)$-probably intersectionally fair  with
\begin{equation}
\epsilon_1=\sqrt{2}\frac{s^*}{\sqrt{\delta}}+ \sup_{\mathcal{Y}}\max\{\gamma+\sum_{k=1}^d \log \Big( \frac{p_{\hat{Y}}}{\inf_{ \mathcal{A}_k}p_{\hat{Y} \mid A_k}} \Big),-\gamma+\sum_{k=1}^d \log \Big( \frac{\sup_{\mathcal{A}_k}p_{\hat{Y} \mid A_k}}{p_{\hat{Y}}} \Big)\}
\end{equation}
\end{theorem*}

The proof for both of these variants is exactly the same as for $\ref{thm.chebyshev}$ until we obtain an upper and lower bound on $\sup p_{\hat{Y} \mid A}$ and $\inf p_{\hat{Y} \mid A}$. We then use the fact that for $E \subset \mathbb{R}$, $\sup_{(x,y)\in E^2} \vert x - y \vert \leq \sup_E x - \inf_E y$, and $\sup_{x\in E} \vert \log(x/y) \vert \leq \max\{ \log(y/\inf_E x), \log(\sup_E x/y)\}$.

\subsection{Comparison between other measures of fairness}\label{app.comparison}

In \cite{Kgerrymandering}, a different fairness metric is used. Indeed, instead of simply measuring unfairness as the unweighted difference $\vert \Pr(\hat{Y}=1 \mid A=a) - \Pr(\hat{Y}=1)\vert$, they use the weighted difference $\Pr(A=a) \vert \Pr(\hat{Y}=1 \mid A=a) - \Pr(\hat{Y}=1)\vert$. For this subsection, we will use as a definition of unfairness $\UFI=\sup_{a \in \mathcal{A}} u(a)$ with $u(a)=\vert \Pr(\hat{Y}=1 \mid A=a) - \Pr(\hat{Y}=1)\vert$. We define the weighted unfairness used in \cite{Kgerrymandering} as $w^*=\sup_{a \in \mathcal{A}} w(a)$ with $w(a)=p_A(a) u(a)$ the weighted version of $u$. This definition yields very useful statistical properties in terms of the unfairness estimation, and \cite{Kgerrymandering} shows with Theorem 2.11 that the error made using the empirical estimator is less than $\tilde{\mathcal{O}}( \sqrt{((1+VCDIM(H)) \log(n)-\log(\delta))/n})$ with high probability $1-\delta$. Unfortunately this notion of unfairness is hard to control as the meaning of $w^* \leq \epsilon$ may be difficult to use for a decision maker, and can lead to the discrimination of groups of small sizes compared to using $\UFI$. This is already discussed and supported empirically in \cite{Fintersectional}.

We will briefly give some inequalities relating these quantities.
We have that 
\begin{equation}
w^*= \sup_{\mathcal{A}} p_Au \leq \sum_{a \in \mathcal{A}} p_A(a) u(a) =\EUFI \leq \UFI.
\end{equation}
Through these equations we see that $w^*$ cannot approach $\UFI$. 

The advantage of the notion of probable intersectional fairness compared to $w^*$ is two-fold: we can be arbitrarily close to $\UFI$, and through $\delta$ we explicitly control the size of the population that faces discrimination.  

Additionally we will present an example, which shows that when the number of protected groups grows large, the notion of weighted unfairness can become inadequate for certain scenarios compared to probabilistic unfairness. 

We will consider that we have $991$ protected groups with three different protected groups sets of sizes $1$, $495$ and $495$, for which we will denote any of their elements by $a_1$, $a_2$, and $a_3$ respectively. We will consider that $\Pr(A=a_1)=0.01$, and that the remaining protected groups are distributed uniformly $\Pr(A=a_2)=\Pr(A=a_3)=0.001$.

Suppose that $\Pr(\hat{Y}=1 \mid A=a_1)=1$ and $\Pr(\hat{Y}=1 \mid A=a_2)=\Pr(\hat{Y}=1 \mid A=a_3)=1/2$. Then $\Pr(\hat{Y}=1)=1/100+99/200=101/200$. Thus $u(a_1)=1-101/200=99/200$, $w(a_1)=99/20000$, $u(a_2)=u(a_3)=101/200-1/2=1/200$, and $w(a_2)=w(a_3)=1/200000$. Which means that $w^* = 99/20000$ and the model is $(1/200,0.99)$-probabilistically fair. Now suppose that $\Pr(\hat{Y}=1 \mid A=a_1)=1$,  $\Pr(\hat{Y}=1 \mid A=a_2)=1$, and $\Pr(\hat{Y}=1 \mid A=a_3)=0$. Then we have $\Pr(\hat{Y}=1)=101/200$. Thus $u(a_1)=99/200$,  $w(a_1)=99/20000$,  $u(a_2)=99/200$, $w(a_2)=99/200000$,  $u(a_3)=101/200$, and $w(a_3)=101/200000$. Which means that $w^* = 99/20000$ and the model is $(101/200,0.99)$-probabilistically fair. Here we see from these two examples, that $99\%$ population of their population saw their unfairness multiply by about a $100$ times while $w^*$ did not change. But probabilistic unfairness did manage to capture this change.

What we see is that when there is a high number of protected groups, relatively bigger groups tend to determine the weighted measure of unfairness $w^*$, but they can still consist of only a very small part of the total population overall.

We present here two simple inequalities relating $(\epsilon,\delta)$ probabilistic fairness with $w^*$ and $\EUFI$. 
\begin{align}
w^* &\leq \max \{ \epsilon\sup_{  \mathcal{A}}p_A, \delta \UFI  \}\\
\EUFI &\leq \epsilon + \delta \UFI
\end{align}

\begin{proof}
Let $\mathcal{A}_{\epsilon}=\{a \in \mathcal{A} \mid u(a)\leq \epsilon \}$ and $\mathcal{A}_{\epsilon}^C$ its complementary set. 

If $a^* =\argmax w(a) \in \mathcal{A}_{\epsilon}$ then $w(a^*)=p_A(a^*)u(a^*)\leq \epsilon \sup_{\mathcal{A}}p_A  $. Otherwise using that $\Pr(\mathcal{A}_{\epsilon}^C)\leq \delta$, we have $w(a)\leq \delta \UFI$. 

Now for $\EUFI$:
\begin{align*}
\EUFI &= \sum_{a \in \mathcal{A}_{\epsilon}} p_A(a) u(a)+ \sum_{a \in \mathcal{A}_{\epsilon}^C} p_A(a) u(a)\\
& \leq \sum_{a \in \mathcal{A}_{\epsilon}} p_A(a) \epsilon+ \sum_{a \in \mathcal{A}_{\epsilon}^C} p_A(a) \UFI \\
&= \epsilon \Pr(\mathcal{A}_{\epsilon}) + \UFI  \Pr(\mathcal{A}_{\epsilon}) \\
& \leq \epsilon + \delta \UFI.
\end{align*}
\end{proof}

\subsection{Intersectional Fairness and Continuous Protected Attributes} \label{app.continuous} 

Here we will show that when $A$ is continuous, even for reasonable distributions, we might end up with $\UFI=\infty$. Whereas our definition of probabilistic unfairness still has finite values, and can therefore be used as an interpretable tool to compare unfairness across models. 

Suppose that we have a random vector $(A_1,A_2,...,A_d,\hat{Y})$ distributed according to a multivariate normal $\mathcal{N}(\mu, \Sigma)$ with $\mu$ and $\Sigma$ the mean and covariance. Because $\hat{Y}$ is continuous, we will instead use the density $f_{\hat{Y} \mid A}$ in the definition of $\UFI$. Because this vector is distributed according to a multivariate normal, the conditional distribution is still normal and we can derive the exact parameters. The conditional distribution is 
\begin{equation}
\hat{Y} \mid A=a \sim \mathcal{N}(\bar{\mu},\bar{\Sigma}) 
\end{equation}
with $\bar{\mu}$ a linear form in $a$, and $\bar{\Sigma}$ that depends only in $\Sigma$. Basically, we can make the mean go to $\infty$ by making $a$ go to $\infty$. Hence for a given $y \in \mathcal{Y}$, we have that $\inf_{a \in \mathcal{A}} f_{\hat{Y}\mid A}(y\mid a)= 0$ for all $y$, which means that the unfairness is always infinite. 

Whereas our notion of probabilistic fairness, is finite and computationally tractable as we need to evaluate $\delta=\mathbb{E}[\mathbbm{1}[U>\epsilon]]$. It goes to $1$ as $\epsilon$ goes to $\infty$.

If we want to compare two machine learning models, and we do not want to compare for a specific point $\delta$, then $\epsilon^*(\delta)$ can be seen as a function of $\epsilon$, and we can compare these functions. If for two models $h_1$ and $h_2$ one function is always above the other, we could say that one is more fair than the other.

As an example, we consider for $d=10$ the couples $(A,\hat{Y}) \sim \mathcal{N}(0,\Sigma_w)$ with $\Sigma_w$ generated through a Wishart distribution, and $(A,\hat{Y}) \sim \mathcal{N}(0,\Sigma_c)$ with $\Sigma_c=(\textbf{1}+I_d)/2$ and $\textbf{1}$ the constant matrix equal to $1$. We then compute the probabilistic fairness on Figure \ref{fig.continuous} by computing the expectation of $\mathbbm{1}[U>\epsilon]$.   

\begin{figure}[ht] 
\begin{center} 
\centerline{\includegraphics[width=1\columnwidth]{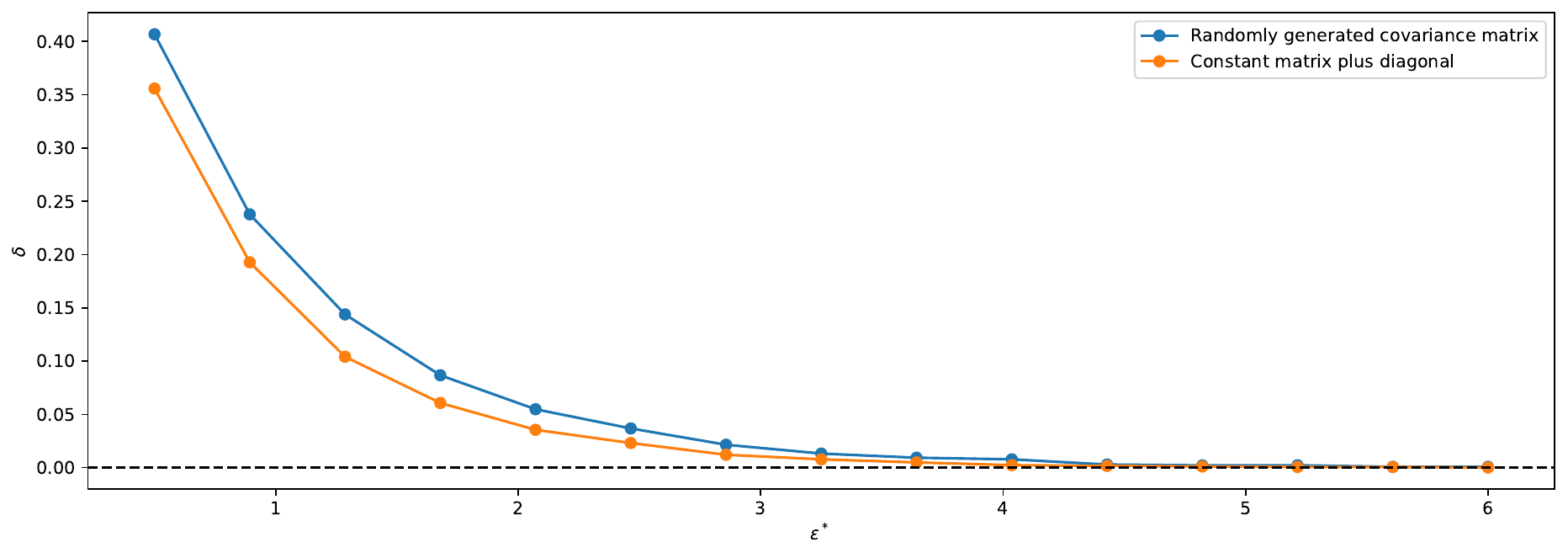}} 
\caption{Probabilistic Unfairness for continuous protected attributes.}\label{fig.continuous}
\end{center}
\end{figure}

There are other fairness metrics specifically for continuous attributes, such as in  \cite{fair_continuous} the HGR coefficient between $A$ and $\hat{Y}$, but which may be less interpretable to decision makers.

\section{Missing proofs and elements of part \ref{sec.bounds}}

\subsection{Counter example with independence of the sensitive attributes}\label{app.counter_example}

Let us define the following probability distribution on $(A_1,A_2,\hat{Y})$, with $A_1$, $A_2$, and $\hat{Y}$ binary:
\begin{align*}
&\frac{3}{16}=\Pr(A_1=0,A_2=0,\hat{Y}=0)\\
&=\Pr(A_1=1,A_2=1,\hat{Y}=0)\\
&=\Pr(A_1=0,A_2=1,\hat{Y}=1)\\
&=\Pr(A_1=1,A_2=0,\hat{Y}=1)\\
\text{and} \quad \frac{1}{16}&=\Pr(A_1=0,A_2=0,\hat{Y}=1)\\
&=\Pr(A_1=0,A_2=1,\hat{Y}=0)\\
&=\Pr(A_1=1,A_2=0,\hat{Y}=0)\\
&=\Pr(A_1=1,A_2=1,\hat{Y}=1)
\end{align*}

We have $p_A=1/4$, and $p_{A_1}=p_{A_2}=1/2$. Therefore $A_1 \indep A_2$. We have $p_{\hat{Y}}=1/2$, hence $p_{A_1\mid \hat{Y}}(0\mid 0)=p_{A_2\mid \hat{Y}}(0\mid 0)=1/2$
and $p_{A_1,A_2\mid \hat{Y}}(0,0,\mid 0)=3/8\neq p_{A_2\mid \hat{Y}}(0\mid 0) p_{A_1\mid \hat{Y}}(0\mid 0)$, therefore the $A_k$ are not independent conditionally on $\hat{Y}$. Because $p_{A_1 \mid \hat{Y}}=p_{A_2 \mid \hat{Y}}=1/2$, any form of marginal unfairness is $0$, and $\UFI=\log((3/4)/(1/4))=\log(3) \neq 0$. In this example we have mutual independence of the $A_k$, independence between $A_k$ and $\hat{Y}$, but still no meaningful relationship between intersectional and marginal fairness because we did not have independence conditionally on $\hat{Y}$.

\subsection{Proof of Proposition \ref{thm.independent}} \label{app.thm_indep}

Using the assumed independence, for any $a$ in $\mathcal{A}$ and $y$ in $\mathcal{Y}$ we can rewrite $p_{\hat{Y}\mid A}$ with marginal quantities:
\begin{align}
\Pr(\hat{Y}=y \mid A=a)&=\frac{\Pr(A=a \mid \hat{Y}=y)\Pr(\hat{Y}=y)}{\Pr(A=a)}\\
&=\Pr(\hat{Y}=y) \frac{\prod_{k=1}^d \Pr(A_i=a_i \mid \hat{Y}=y)}{\prod_{k=1}^d \Pr(A_i=a_i)}\\
&=\Pr(\hat{Y}=y) \prod_{k=1}^d \frac{\Pr(\hat{Y}=y \mid A_i=a_i)}{\Pr(\hat{Y}=y)}.
\end{align}
 Because the numerator is a product of independent variables (in the functional sense), taking the sup in $\mathcal{A}$ yields:
 \begin{equation}
 \sup_{a \in \mathcal{A}} \Pr(\hat{Y}=y \mid A=a) =\Pr(\hat{Y}=y) \prod_{k=1}^d \frac{\sup_{a_k \in \mathcal{A}_k} \Pr(\hat{Y}=y \mid A_k=a_k)}{\Pr(\hat{Y}=y)}.
 \end{equation}
 We can do the same for $\inf$. Hence 
 \begin{equation}
\frac{\sup_{a \in \mathcal{A}} \Pr(\hat{Y}=y \mid A=a)}{\inf_{a \in \mathcal{A}} \Pr(\hat{Y}=y \mid A=a)}=\prod_{k=1}^d \frac{\sup_{a_k \in \mathcal{A}_k} \Pr(\hat{Y}=y \mid A_k=a_k)}{\inf_{a_k \in \mathcal{A}_k} \Pr(\hat{Y}=y \mid A_k=a_k)},
 \end{equation}
 and we obtain 
 \begin{equation}
\UFI=\sup_{y \in \mathcal{Y}} \sup_{(a,a') \in \mathcal{A}^2} \Big \vert \log \Big( \frac{\Pr(\hat{Y}=y \mid A=a)}{\Pr(\hat{Y}=y \mid A=a')}\Big) \Big \vert=\sup_{y \in \mathcal{Y}} \sum_{k=1}^d \sup_{(a_k,a_k') \in \mathcal{A}_k^2} \ufi(y,a_k,a_k').
 \end{equation}
 The inequality is obtained by triangle inequality and because $\sup_{y \in \mathcal{Y}} \sup_{(a_k,a_k') \in \mathcal{A}_k^2} \ufi[k](y,a_k,a_k')=\UFI[k]$ by definition.

\subsection{Proof of Theorem \ref{thm.chebyshev} and Corollary \ref{cor.fraction}} \label{app.chebyshev} 

\begin{theorem*}
For $\delta \in (0,1]$, any classifier $h$ over a distribution $\mathcal{D}$ is $(\epsilon_1,\delta)$ and $(\epsilon_2,\delta)$-probably intersectionally fair  with \vspace{-3mm}
\begin{align*}
&\epsilon_1\!=\!2\sqrt{2} \! \frac{\smax}{\sqrt{\delta}}\!+\!\sup_{y \in \mathcal{Y}} \Big\{ \! \sum_{k=1}^d\! \sup_{(a_k,a_k') \in \mathcal{A}_k^2}\! \ufi[k](y,a_k,a_k')\Big\} \\
&\epsilon_2\!=\!\sqrt{2}\frac{\smax}{\sqrt{\delta}}\!+\!\gamma\!+\!\sup_{y \in \mathcal{Y}} \Big\{\sum_{k=1}^d \!\log \!\Big( \!\frac{p_{\hat{Y}}^{1-1/d}(y)}{\inf_{a_k \in \mathcal{A}_k} p_{\hat{Y}\mid A_k}(y \! \mid \! a_k)} \!\Big) \!\Big\}
   \\
&\text{where} \quad \smax=(\sigma^{2/3}+\sigma_y^{2/3})^{3/2} \quad \text{and} \quad \gamma\!=\!C(A)-C(A\!\mid\!\hat{Y})\!=\!\big(\sum_{k=1}^d I(A_k,\hat{Y})\big)-I(A,\hat{Y}).
\end{align*}
\end{theorem*}

\begin{proof}
We want to show that our classifier is $(\epsilon,\delta)$ probably fair for a given $\delta$.

We will first bound in probability $L$ and $L_y$, to be able to approach the joint densities through the product of marginal densities. We will denote by $\mu$, $\mu_y$, $\sigma$ and $\sigma_y$ the expectations and variances of $L$ and $L_y$
Let us apply Chebyshev's inequality to $L$. We obtain that 
\begin{equation*}
\Pr(\vert L - \mu \vert \geq \alpha_1) \leq \frac{\sigma^2}{\alpha_1^2}.
\end{equation*}
Using the fact that $\{ \vert L - \mu \vert < \alpha_1 \} \subset \{ \vert L - \mu \vert \leq \alpha_1 \}$ and taking the complementary event we can write that 
\begin{equation*}
\Pr(\vert L - \mu \vert \leq \alpha_1) \geq 1- \Pr(\vert L - \mu \vert > \alpha_1) \geq 1-\frac{\sigma^2}{\alpha_1^2}.
\end{equation*}
From this inequality we have
\begin{align*}
\vert L - \mu \vert \leq \alpha_1 & \implies \begin{cases} L - \mu \leq \alpha_1 \\ \mu - L \leq \alpha_1 \end{cases}\\
\implies & \begin{cases} L \leq \alpha_1+ \mu \\
L \geq \mu - \alpha_1 \end{cases} \\
\implies & \begin{cases} p_A(A) \leq e^{\alpha_1+\mu} \prod p_{A_k}(A_k) \\
p_A(A) \geq e^{\mu -\alpha_1} \prod p_{A_k}(A_k) \end{cases}.
\end{align*}
We can do the same for $L_y$ with a parameter $\alpha_2>0$.\newline 

Now we want to consider a condition on the parameters $\boldsymbol{\alpha}=(\alpha_1,\alpha_2)$ so that the probability of the conjunction of the events $\{\vert L - \mu \vert \leq \alpha_1 \}$ and $\{ \vert L_y - \mu_y \vert \leq \alpha_2 \}$ is greater than $1-\delta$. 
For $\delta > 0$, $\alpha_1 >0$ and $\alpha_2 >0$, a sufficient condition is that $\frac{\sigma^2}{\alpha^2}+\frac{\sigma^2_z}{\alpha_2^2} \leq \delta$. We can show this using complementary event and Boole's inequality:
\begin{align*}
&\Pr( \{\vert L - \mu \vert \leq \alpha_1\} \cap \{\vert L_y - \mu_y \vert \leq \alpha_2\}) \\
& \geq \Pr( \{\vert L - \mu \vert \leq \alpha_1\}) + \Pr( \{\vert L_y - \mu_y \vert \leq \alpha_2\}) -1\\
& \geq (1 -\frac{\sigma^2}{\alpha^2})+(1-\frac{\sigma^2_y}{\alpha_2^2}) -1 \\
& \geq 1 - \delta.
\end{align*}
We define $g(\boldsymbol{\alpha})=\frac{\sigma^2}{\alpha_1^2}+\frac{\sigma^2_z}{\alpha_2^2}-\delta$. \\
For any $\boldsymbol{\alpha}$ such that $g(\boldsymbol{\alpha})\leq 0$ we have with probability at least $1-\delta$ that

\begin{align}
p_{\hat{Y} \mid A}(\hat{Y} \mid A)&= \frac{p_{A \mid \hat{Y}}(A \mid \hat{Y})p_{\hat{Y}}(\hat{Y})}{p_A(A)} \\
&\leq p_{\hat{Y}}(\hat{Y}) \frac{e^{\alpha_2+\mu_y}}{e^{\mu-\alpha_1}}   \frac{\prod_{k=1}^d p_{A_k \mid \hat{Y}}(A_k \mid \hat{Y} )}{\prod_{k=1}^d p_{A_k}(A_k)} \\
&\leq p_{\hat{Y}}(\hat{Y}) \frac{e^{\alpha_2+\mu_y}}{e^{\mu-\alpha_1}}   \frac{\prod_{k=1}^d p_{\hat{Y} \mid A_k}(\hat{Y} \mid A_k )}{p_{\hat{Y}}(\hat{Y})^d} \\
&=p_{\hat{Y}}(\hat{Y}) \varphi(\mu_y,\mu)\psi(\boldsymbol{\alpha})f(\hat{Y},A), \label{eq.upper_bound}
\end{align}
where $\varphi(\mu_y,\mu)=e^{\mu_y-\mu}$, $\psi(\boldsymbol{\alpha})=e^{\alpha_1+\alpha_2}$, and $f(y,a)= \prod_{k=1}^d p_{\hat{Y} \mid A_k}(y \mid a_k)/p_{\hat{Y}}(y)^d$. 
Hence by taking the $\sup$ over $a$ and $\inf$ over $\boldsymbol{\alpha}$ on the right hand-side, we obtain 
\begin{equation*}
p_{\hat{Y} \mid A}(A \mid \hat{Y}) \leq p_{\hat{Y}}(\hat{Y}) \varphi(\mu_y,\mu)  \inf_{g(\boldsymbol{\alpha}) \leq 0} \psi(\boldsymbol{\alpha}) \sup_{a \in \mathcal{A}} f(\hat{Y},a).
\end{equation*}

As it is a product of functions of independent variables, $\sup_{a \in \mathcal{A}} f(y,a)$ is just the product of the $\sup$ of each $p_{\hat{Y} \mid A_k}$.

We will now solve the constrained optimization problem for $\psi$. We can write $\inf e^{\alpha_1+\alpha_2}= e^{\inf (\alpha_1+\alpha_2)}$, so we will just need to solve the simpler problem $\inf_{g(\alpha)\leq 0} s(\alpha)$, with $s(\alpha)= \alpha_1+\alpha_2$. Let us compute the gradients of $s$ and $g$:
\begin{align*}
\nabla g &= (-2 \sigma^2 \alpha_1^{-3},-2 \sigma_y^2 \alpha_2^{-3})^{\top}, \\
\nabla s &= (1,1)^{\top}.
\end{align*}
We will now show that this is a convex problem. The function $s$ is linear thus convex, and we will now compute the hessian of $g$:
\begin{align*}
H_g=6 \cdot \begin{pmatrix}
\sigma^2 \alpha_1^{-4} & 0 \\
0 & \sigma_y^2 \alpha_2^{-4} 
\end{pmatrix}. \\
\end{align*}
Clearly we have that the determinant of $H_g$, $Det(H_g)$ is strictly positive. Therefore $H_g$ is definite positive, and $g$ is convex. And for $\delta>0$ there is a feasible interior point by taking $\alpha_1$ and $\alpha_2$ big enough, which means that Slater's conditions hold (e.g. a convex constraint with a feasible interior point). We will now analyze the KKT conditions for minimization with the dual parameter $c\geq 0$: 
\begin{align*}
&\begin{cases}\nabla s + c\nabla g=0 \\
c g(\boldsymbol{\alpha})=0
\end{cases}
\Leftrightarrow \begin{cases} 1= 2 c \sigma^2 \alpha_1^{-3} \\
1=2 c \sigma_y^2 \alpha_2^{-3} \\
c g(\alpha_1,\alpha_2)=0
\end{cases}.
\end{align*}
We obtain that $\alpha_1= \sqrt[3]{2 c \sigma^2}$ and $\alpha_2= \sqrt[3]{2 c \sigma_y^2}$ \\
Clearly $c >0$ otherwise the first two lines cannot be 1, hence using the last equation we have $g(\alpha_1,\alpha_2)=0$. We now develop this last equality to obtain $c$: 
\begin{align*}
&g(\boldsymbol{\alpha})=0  \implies  \frac{\sigma^2}{(2 c \sigma^2)^{2/3}}+\frac{\sigma_y^2}{(2 c \sigma_y^2)^{2/3}} = \delta \\
&\implies  c = \frac{1}{2} \left( \frac{\sigma^{2/3}+\sigma_y^{2/3}}{\delta}\right)^{3/2}.
\end{align*}
Plugging $c$ in the previous expressions we have
\begin{align*}
&\implies \begin{cases} \alpha_1^*=\sqrt{\left( \frac{\sigma^{2/3}+\sigma_y^{2/3}}{\delta}\right)} \sqrt[3]{\sigma^2}=\frac{s_1^*}{\sqrt{\delta}}\\
\alpha_2^*=\sqrt{\left( \frac{\sigma^{2/3}+\sigma_y^{2/3}}{\delta}\right)} \sqrt[3]{\sigma_y^2}=\frac{s_2^*}{\sqrt{\delta}} \end{cases} \\
& \mbox{with} \quad \begin{cases} s_1^*=\sqrt{ \sigma^{2/3}+\sigma_y^{2/3}} \sigma^{2/3} \\ s_2^*=\sqrt{ \sigma^{2/3}+\sigma_y^{2/3}} \sigma_y^{2/3} \end{cases}.
\end{align*}
Finally the minimum is 
\begin{align*}
&\inf_{g(\boldsymbol{\alpha})\leq 0}s=\frac{s^*_1+s^*_2}{\sqrt{\delta}}=\frac{\left(\sigma^{2/3}+\sigma_y^{2/3} \right)^{3/2}}{\sqrt{\delta}}=\frac{s^*}{\sqrt{\delta}} \\
&\mbox{with} \quad s^*=\left(\sigma^{2/3}+\sigma_y^{2/3} \right)^{3/2}.
\end{align*}

We will now do the same in order to lower bound $p_A(\hat{Y} \vert A)$. 
\begin{align}
p_{\hat{Y} \mid A}(\hat{Y} \mid A)&= \frac{p_{A \mid \hat{Y}}(A \mid \hat{Y})p_{\hat{Y}}(\hat{Y})}{p_A(A)} \\
&\geq p_{\hat{Y}}(\hat{Y}) \frac{e^{-\alpha_2+\mu_y}}{e^{\mu+\alpha_1}}   \frac{\prod_{k=1}^d p_{A_k \mid \hat{Y}}(A_k \mid \hat{Y} )}{\prod_{k=1}^d p_{A_k}(A_k)} \\
&\geq p_{\hat{Y}}(\hat{Y}) \frac{e^{-\alpha_2+\mu_y}}{e^{\mu+\alpha_1}}   \frac{\prod_{k=1}^d p_{\hat{Y} \mid A_k}(\hat{Y} \mid A_k)}{p_{\hat{Y}}(\hat{Y})^d} \\
&=p_{\hat{Y}}(\hat{Y}) \varphi(\mu_y,\mu)\psi(\boldsymbol{\alpha})^{-1}f(\hat{Y},A),  \label{eq.lower_bound}
\end{align}
Here we take the $\sup$ over $\boldsymbol{\alpha}$ and $\inf$ over $a$ instead. 
We have that $\sup \psi(\boldsymbol{\alpha})^{-1}=(\inf \psi(\boldsymbol{\alpha}))^{-1}=\exp(-s^*/\sqrt{\delta})$. 

Because $U$ involves the two variables $A$ and $A'$, we need to bound $L'$ and $L'_y$ the variables $L$ and $L_y$ that are taken as a function of $A'$ instead of $A$. Because $(A',\hat{Y}) \sim (A,\hat{Y})$, all the computations above still apply, and we have 

\begin{equation*}
\Pr( \{\vert L - \mu \vert \leq \alpha^*_1\} \cap \{\vert L_y - \mu_y \vert \leq \alpha^*_2\} \cap \{\vert L' - \mu \vert \leq \alpha^*_1\} \cap \{\vert L'_y - \mu_y \vert \leq \alpha^*_2\}) \geq 1-2\delta.
\end{equation*}
Hence we only need to replace $\delta$ by $\delta/2$ in the above inequalities. 

Combining everything, we can conclude that when the event $\{\vert L - \mu \vert \leq \alpha^*_1\} \cap \{\vert L_y - \mu_y \vert \leq \alpha^*_2\} \cap \{\vert L' - \mu \vert \leq \alpha^*_1\} \cap \{\vert L'_y - \mu_y \vert \leq \alpha^*_2\}$ occurs, we have
\begin{align*}
U&=\left \vert \log \left( \frac{p_{\hat{Y} \mid A}(\hat{Y} \mid A)}{p_{\hat{Y} \mid A'}(\hat{Y} \mid A')}\right)\right \vert\\
&\leq \log \left( \frac{p_{\hat{Y}}(\hat{Y}) \varphi(\mu_y,\mu)\exp(\sqrt{2}s^*/\sqrt{\delta})\sup_{\mathcal{A}} f(\hat{Y},a)}{p_{\hat{Y}}(\hat{Y}) \varphi(\mu_y,\mu)\exp(-\sqrt{2}s^*/\sqrt{\delta}) \inf_{\mathcal{A}} f(\hat{Y},a)}\right)\\
&=2\sqrt{2} \frac{s^*}{\sqrt{\delta}}+\log\left( \frac{\prod_{k=1}^d \sup_{a_k \in \mathcal{A}_k} p_{\hat{Y} \mid A_k}(a_k)}{\prod_{k=1}^d \inf_{a_k \in \mathcal{A}_k} p_{\hat{Y} \mid A_k}(a_k)}\right)\\
&\leq 2\sqrt{2} \frac{\smax}{\sqrt{\delta}}\!+\!\sup_{y \in \mathcal{Y}} \! \sum_{k=1}^d\! \sup_{(a,a') \in \mathcal{A}^2}\! \ufi[k](y,a,a').
\end{align*}
We can conclude that 
\begin{align*}
&\Pr(U \leq \epsilon_1) \geq 1-\delta \\
&\mbox{with} \quad \epsilon_1= 2\sqrt{2} \frac{\smax}{\sqrt{\delta}}\!+\!\sup_{y \in \mathcal{Y}} \! \sum_{k=1}^d\! \sup_{(a,a') \in \mathcal{A}^2}\! \ufi[k](y,a,a'),
\end{align*}
which means that our classifier is $(\epsilon_1,\delta)$-probably intersectionally fair. 
Note that $\epsilon_1$ is a function of $(\delta, \sigma, \sigma_y)$.  

In order to derive the proof for $\epsilon_2$, we simply remark that $p_{\hat{Y} \mid A} \leq 1$ which can be used to upper bound the numerator. 
Therefore when the event $\{\vert L - \mu \vert \leq \alpha^*_1\} \cap \{\vert L_y - \mu_y \vert \leq \alpha^*_2\} \cap \{\vert L' - \mu \vert \leq \alpha^*_1\} \cap \{\vert L'_y - \mu_y \vert \leq \alpha^*_2\}$ occurs, we have
\begin{align*}
U&=\left \vert \log \left( \frac{p_{\hat{Y} \mid A}(\hat{Y} \mid A)}{p_{\hat{Y} \mid A'}(\hat{Y} \mid A')}\right)\right \vert\\
&\leq \log \left( \frac{1}{p_{\hat{Y}}(\hat{Y}) \varphi(\mu_y,\mu)\exp(-\sqrt{2}s^*/\sqrt{\delta}) \inf_{\mathcal{A}} f(\hat{Y},a)}\right)\\
&=\!\sqrt{2}\frac{\smax}{\sqrt{\delta}}\!+\!\gamma\!+\!\sum_{k=1}^d \!\log \!\Big( \!\frac{p_{\hat{Y}}^{1-1/d}(\hat{Y})}{\inf_{a_k \in \mathcal{A}_k} p_{\hat{Y}\mid A_k}(\hat{Y} \! \mid \! a_k)} \!\Big) \! \\
&\leq \!\sqrt{2}\frac{\smax}{\sqrt{\delta}}\!+\!\gamma\!+\!\sup_{y \in \mathcal{Y}} \Big\{\sum_{k=1}^d \!\log \!\Big( \!\frac{p_{\hat{Y}}^{1-1/d}(y)}{\inf_{a_k \in \mathcal{A}_k} p_{\hat{Y}\mid A_k}(y \! \mid \! a_k)} \!\Big) \!\Big\}\\
&=\epsilon_2.
\end{align*}
We can conclude in the same way as for $\epsilon_1$.

Looking at the proof, it also holds true that the model will also be  $(\min \{ \epsilon_1, \epsilon_2\},\delta)$-probably intersectionally fair. 
 \end{proof}

Now let us prove Corollary \ref{cor.fraction}:

\begin{corollary*} 
Denoting $(\Omega,\mathcal{T},\Pr)$ the probability space on which $(\hat{Y},A,A')$ is defined, there exists an event $F$ so that for $f(y,a)=\prod_{k=1}^d (p_{\hat{Y} \mid A_k}(y \mid a_k)/p_{\hat{Y}}(y))$ we have for $U$ the random unfairness:
\begin{equation}
-\frac{2 \sqrt{2}\smax}{\sqrt{\delta}} + \sup_{\omega \in F}\left \vert \log\left( \frac{f(\hat{Y},A')}{f(\hat{Y},A)}\right)\!(\omega)\!  \right  \vert \leq \sup_{\omega \in F} U(\omega) \leq \frac{2 \sqrt{2}\smax}{\sqrt{\delta}} + \sup_{\omega \in F}\left \vert \log\left( \frac{f(\hat{Y},A')}{f(\hat{Y},A)}\right)\!(\omega)\!  \right  \vert,
\end{equation}
with $\Pr(F)\geq 1- \delta$. 
\end{corollary*}

\begin{proof}

Let $F$ be the event defined as follows:
\begin{align}
F=\{ \omega \in \Omega \mid p_{\hat{Y} \mid A}(\hat{Y}\mid A)(\omega) \leq  p_{\hat{Y}}(\hat{Y})(\omega) \exp{(\frac{\sqrt{2} \smax}{\sqrt{\delta}} -\gamma)} f(\hat{Y},A)(\omega), \\
p_{\hat{Y} \mid A'}(\hat{Y} \mid A')(\omega) \leq p_{\hat{Y}}(\hat{Y}) (\omega)\exp{(\frac{\sqrt{2} \smax}{\sqrt{\delta}} -\gamma)} f(\hat{Y},A')(\omega), \\
p_{\hat{Y} \mid A}(\hat{Y} \mid A)(\omega) \geq  p_{\hat{Y}}(\hat{Y}) (\omega)\exp{(\frac{-\sqrt{2} \smax}{\sqrt{\delta}} -\gamma)} f(\hat{Y},A)(\omega), \\
p_{\hat{Y} \mid A'}(\hat{Y} \mid A')(\omega) \geq p_{\hat{Y}}(\hat{Y})(\omega) \exp{(\frac{-\sqrt{2} \smax}{\sqrt{\delta}} -\gamma)} f(\hat{Y},A')(\omega)\}.
\end{align}

What we have shown in the above proof of the Theorem before taking the $\sup$ and $\inf$ in $a \in \mathcal{A}$, is that the event  $\{\vert L - \mu \vert \leq \alpha^*_1\} \cap \{\vert L_y - \mu_y \vert \leq \alpha^*_2\} \cap \{\vert L' - \mu \vert \leq \alpha^*_1\} \cap \{\vert L'_y - \mu_y \vert \leq \alpha^*_2\}$ is included in $F$, and hence $\Pr(F)\geq 1-\delta$. Using these upper bounds, and simplifying by $p_{\hat{Y}}(\hat{Y})\exp(-\gamma)$ when taking the ratio of $p_{\hat{Y}\mid A}$ and $p_{\hat{Y}\mid A'}$ ,
we have that 
\begin{equation}
-\frac{2 \sqrt{2}\smax}{\sqrt{\delta}} + \log\left( \frac{f(\hat{Y},A)(\omega)}{f(\hat{Y},A')(\omega)}\right) \leq \log\left( \frac{ p_{\hat{Y} \mid A}(\hat{Y}\mid A)(\omega)}{p_{\hat{Y} \mid A'}(\hat{Y} \mid A')(\omega)} \right) \leq \frac{2 \sqrt{2}\smax}{\sqrt{\delta}} + \log\left( \frac{f(\hat{Y},A)(\omega)}{f(\hat{Y},A')(\omega)}\right), 
\end{equation}
and that 
\begin{equation}
-\frac{2 \sqrt{2}\smax}{\sqrt{\delta}} + \log\left( \frac{f(\hat{Y},A')(\omega)}{f(\hat{Y},A)(\omega)}\right) \leq \log\left( \frac{ p_{\hat{Y} \mid A'}(\hat{Y}\mid A')(\omega)}{p_{\hat{Y} \mid A}(\hat{Y} \mid A)(\omega)} \right) \leq \frac{2 \sqrt{2}\smax}{\sqrt{\delta}} + \log\left( \frac{f(\hat{Y},A')(\omega)}{f(\hat{Y},A)(\omega)}\right), 
\end{equation}
therefore 
\begin{equation}
\left\vert \log\left( \frac{ p_{\hat{Y} \mid A'}(\hat{Y}\mid A')(\omega)}{p_{\hat{Y} \mid A}(\hat{Y} \mid A)(\omega)} \right) \right \vert \leq \frac{2 \sqrt{2}\smax}{\sqrt{\delta}} + \left \vert \log\left( \frac{f(\hat{Y},A')(\omega)}{f(\hat{Y},A)(\omega)}\right)\right \vert, 
\end{equation}
and 
\begin{equation}
-\frac{2 \sqrt{2}\smax}{\sqrt{\delta}} + \left \vert \log\left( \frac{f(\hat{Y},A')(\omega)}{f(\hat{Y},A)(\omega)}\right)\right \vert \leq \left\vert \log\left( \frac{ p_{\hat{Y} \mid A'}(\hat{Y}\mid A')(\omega)}{p_{\hat{Y} \mid A}(\hat{Y} \mid A)(\omega)} \right) \right \vert.
\end{equation}
Taking the $\sup$ over all elements $\omega$ in $F$ we obtain that 
\begin{equation}
-\frac{2 \sqrt{2}\smax}{\sqrt{\delta}} + \sup_{\omega \in F}\left \vert \log\left( \frac{f(\hat{Y},A')}{f(\hat{Y},A)}\right)\!(\omega)\!  \right  \vert \leq \sup_{\omega \in F} U(\omega) \leq \frac{2 \sqrt{2}\smax}{\sqrt{\delta}} + \sup_{\omega \in F}\left \vert \log\left( \frac{f(\hat{Y},A')}{f(\hat{Y},A)}\right)\!(\omega)\!  \right  \vert,
\end{equation}
where $\sup_{\omega \in F} U(\omega)$ is the maximum random unfairness over the event $F$.\\

We cannot recover a form which uses the $u_k$, as here the variable $A$ and $A'$ are linked through $F$. Notice that we can upper bound  $\sup_{\omega \in F} \vert \log( f(\hat{Y},A')/f(\hat{Y},A)) (\omega)\vert$ by  $\sup_{\omega \in \Omega} \vert \log( f(\hat{Y},A')/f(\hat{Y},A))(\omega)  \vert $ because $F\subset \Omega$, which then recovers Theorem \ref{thm.chebyshev} but stated in a slightly different way. The convenient part of this Corollary is that we are also able to lower bound $\sup_{\omega \in F} U(\omega)$, but of course the inconvenience is that $F$ is unknown. 
\end{proof}

\subsection{Additional Bounds on Probable Intersectional Fairness} \label{app.chernoff}

Looking at how we proved Theorem $\ref{thm.chebyshev}$, we can derive more bounds by using other concentration inequalities. Let $\kappa(t)=\log(\mathbb{E}[e^{tL}])$ and $\kappa_y(t)=\log(\mathbb{E}[e^{tL_y}])$ be the cumulant generating-function of $L$ and $L_y$. We define the $\alpha$-Renyi Divergence between two discrete distributions $P$ and $Q$ of size $S$ for $\alpha>0$ as
\begin{equation}
D_{\alpha}(P \Vert Q)=\frac{1}{\alpha-1} \log( \sum_{k=1}^S \frac{p_k^{\alpha}}{q_k^{\alpha-1}}).
\end{equation}
These moments generating functions can be expressed as functions of Renyi Divergences, indeed
\begin{align}
&\kappa(t)=tD_{t+1}(p_A \Vert \bigotimes_{k=1}^d p_{A_k}) \\
\text{and} \quad &\kappa_y(t)=\log(\sum_{y \in \mathcal{Y}} p_{\hat{Y}}(y) \exp(t D_{t+1}(p_{A \mid \hat{Y}}(\cdot \mid y) \Vert \bigotimes_{k=1}^d p_{A_k \mid \hat{Y}}(\cdot \mid y)))).
\end{align}
While $\kappa(t)$ can therefore be estimated using techniques of \cite{minimax_kl} for instance, the estimation of $\kappa_y$ is less straightforward. 

For $\lambda^+$ and $\lambda^-$ in $\mathbb{R}$, we define $I_y^+(\lambda^+)=\sup_{t \in \mathbb{R}^+} \{t\lambda^+ - \kappa_y(t) \}$ and $I^-(\lambda^-)=\sup_{t \in \mathbb{R}^{-}} \{t\lambda^- - \kappa(t) \}$. We apply the generic Chernoff bounds to $L$ and $L_y$:
\begin{align}
&\Pr(L_y \geq \lambda^+) \leq e^{-I_y^+(\lambda^+)}\\
\mbox{\text{and}} \quad &\Pr(L\leq \lambda^-) \leq e^{-I^-(\lambda^-)}.
\end{align}

We will apply the same reasoning used in Appendix $\ref{app.chebyshev}$ and will only highlight the differences. 
\begin{equation*}
\Pr([L_y \geq \lambda^+] \cap [L\leq \lambda^-]) \geq 1 -(e^{-I_y^+(\lambda^+)}+e^{-I^-(\lambda^-)})
\end{equation*}
We want to ensure that the right hand-side is greater than $1-\delta$. Because we need to make sure that it also holds for $L'$ and $L_y'$, we need $2\delta$ instead of $\delta$. We define the constraint
\begin{equation}
g(\lambda^+,\lambda^-)=e^{-I_y^+(\lambda^+)}+e^{-I^-(\lambda^-)}-2\delta,   
\end{equation}
we need to have $g(\lambda^+,\lambda^-)\leq 0$. Hence using feasible values of $\lambda^+$ and $\lambda^-$, the event $[L_y \geq \lambda^+] \cap [L\leq \lambda^-]$ implies 
\begin{align}
U &\leq (\lambda^+ - \lambda^-)+\sup_{\mathcal{Y}} \sum_{k=1}^d \log(\frac{p_{\hat{Y}}^{1-1/d}}{\inf_{\mathcal{A}_k}p_{\hat{Y} \mid A_k}})\\
\mbox{thus} \quad U&\leq \inf_{g(\boldsymbol{\lambda}) \leq 0}(\lambda^+ - \lambda^-)+\sup_{\mathcal{Y}} \sum_{k=1}^d \log(\frac{p_{\hat{Y}}^{1-1/d}}{\inf_{\mathcal{A}_k}p_{\hat{Y} \mid A_k}})
\end{align}

Using that $\sup_{\mathcal{A}} p_{\hat{Y} \mid A} \leq 1$ we therefore have the following Theorem:
\begin{theorem}
For $\delta \in (0,1]$, any classifier $h$ over a distribution $\mathcal{D}$ is $(\epsilon_3,\delta)$-probably intersectionally fair, with
\begin{equation}
\epsilon_3= \inf_{g(\boldsymbol{\lambda}) \leq 0}(\lambda^+ - \lambda^-)+\sup_{\mathcal{Y}} \sum_{k=1}^d \log(\frac{p_{\hat{Y}}^{1-1/d}}{\inf_{\mathcal{A}_k}p_{\hat{Y} \mid A_k}})
\end{equation}
\end{theorem}

Compared to using Chebyshev, this should be both tighter as we are using more information than the first and second moment, and this bound should also be more efficient in terms of probability, as we are using one sided concentration inequalities. 

\subsection{Properties of the cumulant generating-function}

We now want to be able to say a bit more on the properties of this constrained minimization problem. We will show by first recalling and developing useful properties on $\kappa$ that the problem is non convex and differentiable almost everywhere. 

We will list the properties about $\kappa$ that will be useful to us developed in \cite{dembo2009large}, with \cite{large_deviation} being a summary containing all the information needed.

\begin{lemma} \label{app.conv_lemma}
We have that $\kappa$ is strictly convex and infinitely many times differentiable. This means that $\kappa'$ is strictly increasing and that it can be inverted on $\Image(\kappa')$. 
We will write $\kappa'^{(-1)}=\eta$, and for all $x$ so that $\kappa''(\eta(x))\neq 0$ we have $\eta'=1/\kappa''(\eta)$. 

We also have that $\kappa(0)=0$, $\kappa'(0)=\mu$, and $\kappa''(0)=\sigma^2$. 
\end{lemma}

From these properties we can conclude that $\eta(\mu)=0$, and that $\eta'(\mu)=1/\kappa''(\eta(\mu))=1/\sigma^2$. 

We will recall the definition of the convex conjugate of a function. 

\begin{definition}
Let $\mathcal{E}$ be a Euclidean vector space with scalar product $\langle \cdot, \cdot \rangle$, we define the convex-conjugate of a function $f:\mathcal{E}\to \mathbb{R}$ for $x \in \mathcal{E}$ by
\begin{equation}
f^{\star}(x)=\sup_{t \in \mathcal{E}} \{\langle x, t\rangle-f(t) \}.
\end{equation}
\end{definition}

We will now list the useful properties about $I=\kappa^{\star}$ also developed in \cite{large_deviation}. 

\begin{lemma}
The function $I$ is infinitely many times differentiable on $\Image(\kappa')$, $I(\mu)=0$, and for every $x \in \Image(\kappa')$ we can rewrite $I$ as 
\begin{equation}
I(x)=x\eta(x)-\kappa(\eta(x)).
\end{equation}
\end{lemma}

\begin{proposition} \label{prop.I_continuously_differentiable}
The function $I^+$ is continuously differentiable on $(-\infty,\sup \kappa')$.
\end{proposition}

\begin{proof}
We define $\tilde{I}^+(x)=I(x)$ if $x\geq \mu$, and $\tilde{I}^+(x)=0$ otherwise. 
We will first show that $\tilde{I}^+=I^+$. 

We have $\mu \in \Image(\kappa')$ because $\kappa'(0)=\mu$. Let $f_x(t)=tx-\kappa(t)$.
If $x<\mu$, then because $\kappa'$ is increasing we have for any $t\geq 0$
\begin{align*}
f_x'(t)&=x-\kappa'(t) \\
&\leq \mu - \kappa'(0) \\
&=0.
\end{align*}
This means that the $\max$ on $\mathbb{R}^+$ is at $0$, and therefore $I^+(x)=f_x(0)=0x-\kappa(0)=0=\tilde{I}^+(x)$.

If $x \geq \mu$, then for any $t \leq 0$ we have
\begin{align*}
f_x'(t)&=x-\kappa'(t) \\
&\geq \mu - \kappa'(0) \\
&=0.
\end{align*}
Hence for all $t \leq 0$ we have $f_x(t)\leq f_x(0)$ therefore the sup of $f_x$ is not on $\mathbb{R}^-$. Consequently when $x \geq \mu$ we have $I^+(x)=\tilde{I}^+(x)$. 
All in all we can conclude that $\tilde{I}^+=I^+$ on $\Image(\kappa')$. 

Let us analyze the potential discontinuity at $\mu$. We have $I(\mu)=0$ and $I^+=0$ on $(-\infty,\mu)$, so the function is continuous on $(-\infty,\sup \kappa')$. Let us compute $I'$:
\begin{align*}
I'(x)&=\eta(x)+x\eta'(x)-\eta'(x)\kappa'(\eta(x))\\
&=\eta(x)+x\eta'(x)-\eta'(x)x\\
&=\eta(x),
\end{align*}
and we know that $\eta(\mu)=0$. Hence we have that $I^+$ is continuously differentiable on $(-\infty,\sup \kappa')$.
\end{proof}

\begin{proposition}
The function $e^{-I^+}$ is non-convex at $\mu$. 
\end{proposition}

\begin{proof}
We will simply look at the second derivative of $h(x)=e^{-I^+(x)}$ for $x\geq \mu$:
\begin{align*}
h''(x)&=(I^{+'}(x)^2-I^{+''}(x))e^{-I(x)} \\
&=(\eta(x)^2-\frac{1}{\kappa''(\eta(x))})e^{-I(x)}. 
\end{align*}
Therefore $h''(\mu)=-1/\sigma^2<0$ for $\sigma \neq 0$, which means that it is non-convex at $\mu$. 
\end{proof}

The same proposition applies to $I_y^+$ and $I^-$, with the relevant $\kappa$ or $\kappa_y$. This means that $g''((\mu_y,\mu))=-(1/\sigma^2+1/\sigma_y^2)<0$ hence the following corollary
\begin{corollary}
The constraint $g$ is non convex at $(\mu_y,\mu)$.
\end{corollary}

Finally we remark that when $\mathcal{A}$ is finite, the sup of $\kappa'$ is bounded and therefore there are finite values of $\lambda$ for which $I(\lambda)=\infty$. Hence there are feasible points for any $\delta \in [0,1]$.

\subsection{Errors bounds on $Q$}\label{app.plogp2}

Let $P=(p_1,...,p_S)$ be a discrete distribution of size $\vert P \vert=S$, we want to estimate the quantity $Q(P)=\sum_{k=1}^S p_k \log^2(p_k)$ with $n$ \ac{iid} realizations of $P$. We denote $N_k$, the number of realizations for category $k$. 

In order to bound the $L_2$ error of $\hat{Q}=\sum_{k=1}^{S} (N_k/n) \log^2(N_k/n)$, we will use the bias variance decomposition of $\hat{Q}$:
\begin{align}
&\mathbb{E}[(\hat{Q}-Q)^2]=b(\hat{Q})^2+\Var(\hat{Q})\\
\text{where} \quad  &b(\hat{Q})=\mathbb{E}[\hat{Q}]-Q, \Var(\hat{Q})=\mathbb{E}[(\hat{Q}-\mathbb{E}[\hat{Q}])^2]
\end{align}

The analysis of these error terms is completely derived from \cite{discrete_functionals}. In particular, the method they use for entropy is close to this problem. They show that the bias term can be bounded by deriving smoothness modulus for the function $x \mapsto x \log(x)$, and that the variance term can be bounded using an Efron-Stein inequality. Here, we need to analyse $ x \mapsto x \log^2(x)$, which is technically more difficult as some nice properties such as the convexity of $x \log(x)$ is lost. Still we can show the following two lemmas with the proof further down: 

\begin{lemma} \label{lemma.variance}
\begin{equation}
\Var(Q(\hat{P}_n))= \mathcal{O}\Big(\frac{\log^4(n)}{n}\Big). 
\end{equation}
\end{lemma}

\begin{lemma} \label{lemma.bias}
\begin{equation}
b(\hat{Q})^2=\mathcal{O}\Big( \Big(\frac{\vert P \vert \log(n)}{n} \Big)^2 \Big). 
\end{equation}
\end{lemma}

Using these two lemmas, we can directly conclude that $\mathbb{E}[(\hat{Q}-Q)^2]=\mathcal{O}(\log^4(n)/n)$.

Note that we will directly use some elements already derived in \cite{discrete_functionals}, and will only show here the parts where special care is needed.  

\begin{proof}[Proof of Lemma \ref{lemma.variance}]

Let $f: x \mapsto x \log^2(x)$. \cite{discrete_functionals} analyze the statistic of the form $F(P)=\sum_{k=1}^S f(p_k)$. 
We apply Lemma $13$ of \cite{discrete_functionals} for discrete functionals of $P$, which is derived from a corollary of the Efron-Stein inequality, on $Q$:
\begin{equation}
\Var(\hat{Q}) \leq n \max_{0\leq j \leq n} ( f(\frac{j+1}{n}) - f(\frac{j}{n}))^2.
\end{equation}
We will look for $n$ in $\mathbb{N}^*$ at the function $g : x \mapsto \frac{x+1}{n} \log^2(\frac{x+1}{n}) - \frac{x}{n} \log^2(\frac{x}{n})$ for $x \in [0,n]$. We have
\begin{align}
&g(x)=\frac{x}{n} \log(\frac{x+1}{x}) (\log(x(x+1))-\log(n^2))+\frac{1}{n}\log^2(\frac{x+1}{n}) \\
&=\!\frac{x}{n} \!\log(1\!+\!\frac{1}{x}) (\log(x(x\!+\!1))\!-\!\log(n^2))+\frac{1}{n}\!(\log^2(x\!+\!1) \!+\! \log^2(n) \!-\! 2 \log(x\!+\!1) \log(n)). \label{eq.var_last}
\end{align}
We first look at the term in $1/n \cdot h(x)$. We have that $h'(x)=2(\log(x+1)-\log(n))/(x+1)$ which is $0$ for $x=n-1$. Thus $\argmax_{x \in [0,n]} \vert h(x) \vert \in \{0, n-1, n\}$. We evaluate $h$ at these points: $h(n-1)=0$, $h(n)=\log^2(1+1/n) \sim 1/n^2$, and $h(0)=\log(n)^2$. Hence the max of $\vert h(x) \vert$ over $[0,n]$ is $\log^2(n)$
Using on \eqref{eq.var_last} the inequality $\log(1+1/x)\leq 1/x$, and because $x \mapsto \log(x(x+1))$ is increasing over $\mathbb{R}^+$, we obtain 
\begin{equation}
\vert g(x) \vert\leq \frac{1}{n} (\log(n(n+1))+\log(n^2)) + \frac{1}{n} \log^2(n)=\mathcal{O} \Big( \frac{\log^2(n)}{n} \Big)
\end{equation}
Finally
\begin{equation}
\Var(\hat{Q}) = n \mathcal{O} \Big( \frac{\log^4(n)}{n^2} \Big) =\mathcal{O} \Big( \frac{\log^4(n)}{n} \Big).
\end{equation}
\end{proof}

\begin{proof}[Proof of Lemma \ref{lemma.bias}]

In order to bound the bias \cite{discrete_functionals} use the fact that for any $f$,
\begin{align}
&\mathbb{E}[\hat{Q}]-Q=\sum_{k=1}^S (B_n(f)(p_k) - f(p_k))\\
\text{with} \quad &B_n(f)(x)= \sum_{i=1}^n f(\frac{i}{n}) \binom{i}{n}x^i (1-x)^{n-i}.
\end{align}
The function $B_n(f)(x)$ is the Bernstein polynomial of $f(x)$. Lemma $5$. of \cite{discrete_functionals} shows that for $\varphi(x)=\sqrt{x(1-x)}$:
\begin{align}
&\vert \mathbb{E}[B_n(f)(x)] -f(x)\vert \leq \frac{5}{2}\omega^2_\varphi(f,n^{-1/2})\\
\text{with} \quad & \omega^2_\varphi(f,t)\!=\!\sup\{\vert \! f(u)\!+\!f(v)\!-\!2f(\frac{u+v}{2}) \! \vert , (u,v) \in [0,1]^2, \vert u-v \vert \!\leq\! 2t\varphi(\frac{u\!+\!v}{2})\!\}.
\end{align}
The quantity $\omega^2_\varphi$ is the second-order Ditzian-Totik modulus of smoothness of $f$. It is shown in Lemma $8$ of \cite{discrete_functionals} that for the function $x \mapsto x \log(x)$ (which corresponds the the entropy), $\omega^2_\varphi(f,t)=t^2 \log(4)/(1+t^2)$. We will use a proof similar to Lemma 8 to derive the modulus of smoothness for $x \mapsto x^2 \log(x)$.

In addition, we remark using the triangle inequality that for any $f$ and $g$ continuous functions on $[0,1]$, then 
\begin{equation}
\omega^2_\varphi(f+g,t) \leq \omega^2_\varphi(f,t)+\omega^2_\varphi(g,t).
\end{equation}
Let $f: x \mapsto x \log^2(x)$ for $x$ in $[0,1]$. By expanding $g(x)=x \log^2(x/e) = x \log^2(x)+ x - 2x \log(x)$ we see that we can rewrite $f$ as $f(x)=x \log^2(x/e)+2x \log(x) -x$. Therefore using the previous remark and because $\omega^2_\varphi(x \mapsto x,t)=0$ we have 
\begin{equation}
\omega^2_\varphi(f,t) \leq \omega^2_\varphi(g,t)+2\omega^2_\varphi(f,t)+\omega^2_\varphi(x \mapsto x,t)\leq \omega^2_\varphi(f,t)+\frac{2t^2 \log(4)}{1+t^2}. \label{eq.modulus_bound}
\end{equation}
It remains to upper bound $\omega^2_\varphi(g,t)$. 

First we will show that $g$ is concave on $(0,1]$. We compute the first and second derivative of $g$ for $x$ in $(0,1]$:
\begin{align}
g'(x)&=2 \log(\frac{x}{e}) + \log^2 (\frac{x}{e})=\log^2(x)-1,\\
\text{thus} \quad g''(x)&=2\frac{\log(x)}{x} < 0.
\end{align}
Hence $g$ is strictly concave over $(0,1]$.

Now we will upper bound $\omega^2_\varphi(g,t)$. Let $t$ in $[0,1/2]$. \emph{To be clear, we will use the same language and logic developed in \cite{discrete_functionals} for the proof of Lemma 8 to make the comparison easier}. 
Defining $M=(u+v)/2 \in [0,1]$, then the computation of the second order modulus is an optimization over the regime $\vert u-v \vert \leq 2t \sqrt{M(1-M)}$. Equivalently, it is in the interval $[M (1- \Delta), M(1 + \Delta)] \cap [0,1]$, where $\Delta=t \sqrt{(1-M)/M}$. Because $g$ is strictly concave on $[0,1]$, the maximum of $\vert g(x)+g(y)-2g((x+y)/2) \vert$  is reached at the boundaries of the above feasible interval. We have 
\begin{align}
M(1-\Delta) \geq 0 \Leftrightarrow M \geq \frac{t^2}{1+t^2}, \\
M(1+\Delta) \leq 1 \Leftrightarrow M \leq \frac{1}{1+t^2}.
\end{align}
Therefore the optimization problem defined by the second order modulus of smoothness is equivalent to the maximization of $h(u,v)=\vert g(u)+g(v)-2g((u+v)/2) \vert$ over three different regimes:
\begin{align}
&\text{Regime A:} \quad u=0,v=2M,0\leq M \leq \frac{t^2}{1+t^2}\\
&\text{Regime B:} \quad u=2M-1,v=1, 1 \geq M \geq \frac{1}{1+t^2}\\
&\text{Regime C:} \quad u=M(1+\Delta),v=M(1-\Delta),M\in [\frac{t^2}{1+t^2}, \frac{1}{1+t^2}].
\end{align}

Over the regime $A$: 
\begin{align*}
h(u,v)&=\vert 2M \log^2(\frac{2M}{e})-2M \log^2(\frac{M}{e}) \vert \\
&=2M \vert(\log(\frac{2M}{e})-\log(\frac{M}{e})) (\log(\frac{2M}{e})+\log(\frac{M}{e})) \vert \\
&=2M \log(2) \log(\frac{1}{2} (\frac{e}{M})^2) \\
&=2M \log(2) (2-\log(2)- 2\log(M)). \\
\end{align*}
The function $M \mapsto M$ reaches its max over regime $A$ at $t^2/(1+t^2)$. The function $x \mapsto -x \log(x)$ is positive increasing until $x=1/e$. Hence because $t^2/(1+t^2) \leq 1/e$ for $t \in [0,1/2]$, $M \mapsto - M \log(M)$ reaches its max over regime $A$ also at $t^2/(1+t^2)$. Thus over regime $A$
\begin{equation}
h(u,v)\leq \frac{2t^2}{1+t^2}\log(2)(2-\log(2)+2\log(1+ \frac{1}{t^2})) =_{t \rightarrow 0} \frac{-8\log(2)t^2\log(t)}{1+t^2}+o(-t^2\log(t))
\end{equation}

Over the regime $B$:
\begin{align*}
&h(u,v)=\vert (2M-1) \log^2(\frac{2M-1}{e})+1-2M\log^2(\frac{M}{e}) \vert \\
&=\vert (2M\!-\!1)(\log^2(2M\!-\!1)+1\!-\!2 \log(2M\!-\!1))+1-2M(\log^2(M)+1-2 \log(M))
\vert \\
&=\vert 2M(\log^2(2M-1)- \log^2(M))+4M(\log(M)-\log(2M-1))+2\log(2M-1)-\log^2(2M-1)  \vert \\
&\leq \vert 2M(\log^2(2M-1)- \log^2(M))\vert \!+\! \vert 4M(\log(M)-\log(2M-1))\vert \!+\! \vert 2\log(2M-1) -\log^2(2M-1) \vert \\
&=\vert 2M\log(\frac{2M-1}{M})\log(M(2M-1))\vert \!+\! \vert 4M(\log(\frac{M}{2M-1})\vert \!+\! \vert \log(2M-1)(2-\log(2M-1)) \vert
\end{align*}
We will upper bound each of those three terms. 
First note that as $t \in [0,1/2]$, $t^2/(1-t^2)\leq 1/3$ and $(1+t^2)^2/(1-t^2)\leq 25/12$
Because $-\log$ is decreasing and $2M-1\leq1$, we have that 
\begin{equation*}
\vert  \log(2M-1) \vert=- \log(2M-1) \leq  \log(\frac{1+t^2}{1-t^2})= \log(1+\frac{2t^2}{1-t^2}) \leq \frac{2t^2}{1-t^2}.
\end{equation*}
Hence 
\begin{equation}
\vert \log(2M-1)(2-\log(2M-1)) \vert \leq \frac{t^2}{1-t^2} \frac{16}{3}
\end{equation}

For $M \leq 1$, we have that $(2M-1)/M \leq 1$. Hence 
\begin{equation*}
\vert 4M\log(\frac{M}{2M-1})\vert=4M\log(\frac{M}{2M-1}) \leq 4 \log(\frac{M}{2M-1}) \leq 4 \log(1+\frac{t^2}{1-t^2}) \leq \frac{4t^2}{1-t^2}.
\end{equation*}
The functions $M \mapsto 1/(M(2M-1))$ and $M \mapsto M/(2M-1)$ are decreasing in $M$ over $(1/2,1]$ and bigger than $1$. Therefore 
\begin{align*}
 &\vert \log(M(2M-1)) \vert = \log(\frac{1}{M(2M-1)})\leq \log(\frac{(1+t^2)^2}{1-t^2})\leq\log(\frac{25}{12})\\
 \text{and} \quad &\vert \log(\frac{2M-1}{M}) \vert=\log(\frac{M}{2M-1}) \leq \log(\frac{1}{1-t^2})\leq \frac{t^2}{1-t^2}
\end{align*}
Combining everything we have over regime $B$ using that $M\leq 1$
\begin{equation}
h(u,v)\leq \frac{(4+16/3+\log(25/12))t^2}{1-t^2} =_{t \rightarrow 0} o( - t^2 \log(t))
\end{equation}

Over the regime C:
\begin{align*}
&h(u,v)=M\vert (1-\Delta)\log^2(\frac{M}{e}(1-\Delta))+(1+\Delta)\log^2(\frac{M}{e}(1+\Delta))-2\log^2(\frac{M}{e})   \vert \\
&=\!M \vert (1\!-\!\Delta) \log^2(1\!-\!\Delta)\!+\!(1\!+\!\Delta)\log^2(1\!+\!\Delta) \!-\! 2 \log(\frac{M}{e})( (1\!-\!\Delta) \log(1\!-\!\Delta)\!+\! (1\!+\!\Delta) \log(1\!+\!\Delta)) \vert \\
&=\!\frac{t^2}{t^2+\Delta^2} \vert (1\!-\!\Delta) \log^2(1\!-\!\Delta)\!+\!(1\!+\!\Delta)\log^2(1\!+\!\Delta) \!-\! 2 \log(\frac{M}{e})( (1\!-\!\Delta) \log(1\!-\!\Delta)\!+\! (1\!+\!\Delta) \log(1\!+\!\Delta)) \vert \\
 &\leq \!\frac{t^2}{\Delta^2} \vert (1\!-\!\Delta) \log^2(1\!-\!\Delta)\!+\!(1\!+\!\Delta)\log^2(1\!+\!\Delta)\vert \!+\! \frac{2(1-\log(M))t^2}{\Delta^2} \vert (1\!-\!\Delta) \log(1\!-\!\Delta)\!+\! (1\!+\!\Delta) \log(1\!+\!\Delta) \vert \\
  &\leq \!\frac{t^2}{\Delta^2} \vert (1\!-\!\Delta) \log^2(1\!-\!\Delta)\!+\!(1\!+\!\Delta)\log^2(1\!+\!\Delta)\vert \!+\! \frac{2(1+\log(1+\frac{1}{t^2}))t^2}{\Delta^2} \vert (1\!-\!\Delta) \log(1\!-\!\Delta)\!+\! (1\!+\!\Delta) \log(1\!+\!\Delta) \vert 
\end{align*}
The functions $g_1 : \Delta \mapsto((1+\Delta)\log(1+\Delta)+(1-\Delta)\log(1-\Delta))/\Delta^2$ and $g_2 : \Delta \mapsto((1+\Delta)\log^2(1+\Delta)+(1-\Delta)\log^2(1-\Delta))/\Delta^2$ are both continuous over $[0,1]$ hence bounded with a max reached respectively (can be seen graphically, or by looking at the derivative) at $1$ and $0$. With $g_1(1)=2\log(2)$, and for $\Delta \rightarrow 0$:
\begin{equation}
g_2(\Delta)=\frac{(1-\Delta)(\Delta^2+\Delta^3+o(\Delta^3))+(1+\Delta)(\Delta^2-\Delta^3+o(\Delta^3))}{\Delta^2}\longrightarrow_{\Delta \rightarrow 0}2.
\end{equation}
Finally
\begin{equation}
h(u,v) \leq_{t \rightarrow 0} -8\log(t)t^2 +o(-\log(t)t^2).
\end{equation}

Hence using these bounds over regime $A$, $B$, and $C$, and remarking that it is reached for regime $A$ on $t^2/(1+t^2)$, we obtain
\begin{equation}
\omega^2_\varphi(g,t)=_{t \rightarrow 0} \frac{-8\log(2)t^2\log(t)}{1+t^2}+o(-t^2\log(t)).
\end{equation}
By applying on Equation \eqref{eq.modulus_bound} the upper bounds we derived and taking $t=n^{-1/2}$, we can conclude
\begin{align}
\vert b(\hat{Q}) \vert &\leq \sum_{k=1}^S\vert \mathbb{E}[B_n(f)(p_k)]  -f(p_k)\vert \\ 
&\leq S\frac{5}{2}\omega^2_\varphi(f,n^{-1/2}) \\
& S \leq_{n \rightarrow \infty} 10 \log(2) \frac{\log(n)}{n} +o(\frac{\log(n)}{n})\\
& =_{n \rightarrow \infty} \mathcal{O} \Big( \frac{S\log(n)}{n}\Big) 
\end{align}

\end{proof}

\section{Using partitions of protected attributes}

\subsection{Consistency of $u_{I}^{(q^*)}$}\label{app.consistency}

The main idea of this proof, is that when the number of samples $n$ increases, the probability that $\min_{\mathcal{A},\mathcal{Y}} N_{a,y}< \tau$ goes to $0$ as $n \rightarrow \infty$. And when $\min_{\mathcal{A},\mathcal{Y}} N_{a,y} \geq \tau$ then by definition $\UFIind[(q^*)]=\hat{u}^*$ which is consistent.

We define for $a$ in $\mathcal{A}$ the modified empirical estimator $\hat{p}_A(A=a)$, with $\hat{p}_A(A=a)=N_a/n$ if $N_a>0$ and $1$ otherwise. Using Chebyshev's inequality and because $N_{a,y} \sim \mathcal{B}(p_{A,\hat{Y}}(a,y),n)$ we have for $\epsilon >0$:
\begin{align}
&\Pr(\vert \hat{p}_A(a) \!-\! p_A(a) \vert \! \geq \! \epsilon) \!=\!\Pr(\vert \hat{p}_A(a) \!-\! p_A(a) \vert \! \geq \! \epsilon, N_a \!=\!0)\!+\!\Pr(\vert \hat{p}_A(a) \!-\! p_A(a) \vert \! \geq \! \epsilon, N_a \!>\!0) \\
&\leq \Pr(N_a=0) + \Pr(\vert \frac{N_a}{n} - p_A(a) \vert \geq \epsilon, N_a >0)\\
&\leq (1-p_A(a))^n + \Pr(\vert N_a - n p_A(a) \vert \geq n\epsilon) \\
& \leq (1-p_A(a))^n  + \frac{p_A(a)(1-p_A(a))}{n\epsilon^2} \\
& \longrightarrow_{n \rightarrow \infty} 0,
\end{align}
which means that $\hat{p}_{A}$ is a consistent estimator of $p_{A}$. By Slutsky's Theorem and because $p_A(a)>0$, $\hat{p}_{A,\hat{Y}}(a,y)/\hat{p}(a)$ is a consistent estimator of $p_{\hat{Y} \mid A} (y \mid a)$. Hence by the Continuous Mapping Theorem using the continuous functions $\max$, $\min$, $\log$ and $\vert \cdot \vert$, we have that $\hat{\UFI}$ the estimator using the modified empirical probabilities, is a consistent estimator of $\UFI$. 

Now for the consistency of $\UFIind[(q^*)]$, for $\tau>0$ and $\epsilon>0$, we have 
\begin{align}
\Pr(\vert \UFIind[(q^*)] - \UFI \vert > \epsilon) &= 1 - \Pr(\vert \UFIind[(q^*)] - \UFI \vert \leq \epsilon) \\
& \leq  1 - \Pr(\vert \UFIind[(q^*)] - \UFI \vert \leq \epsilon , \min_{\mathcal{A},\mathcal{Y}} N_{a,y} > \tau)\\
& =  1 - \Pr(\vert \hat{u} - \UFI \vert \leq \epsilon , \min_{\mathcal{A},\mathcal{Y}} N_{a,y} > \tau)\\
& = \Pr( [\vert \hat{u} - \UFI > \leq \epsilon] \cup [\min_{\mathcal{A},\mathcal{Y}} N_{a,y} \leq \tau])\\
& \leq \Pr(\vert \hat{u} - \UFI \vert > \epsilon) + \Pr(\min_{\mathcal{A},\mathcal{Y}} N_{a,y} \leq \tau)\\
& \leq \Pr(\vert \hat{u} - \UFI \vert > \epsilon) + \Pr(\exists (a,y) \in \mathcal{A} \times \mathcal{Y}, N_{a,y} \leq \tau)\\
&  \leq \Pr(\vert \hat{u} - \UFI \vert > \epsilon) + \sum_{(a,y) \in \mathcal{A} \times \mathcal{Y}} \Pr(N_{a,y} \leq \tau).
\end{align}
The first term goes to $0$ by the consistency of $\hat{u}$. We will show that the second term also goes to zero as $n\rightarrow \infty$. 
For $\tau>0$ we apply Hoeffding's inequality on $N_{a,y} \sim \mathcal{B}(p_{A,\hat{Y}}(a,y),n)$ to obtain the following concentration inequality:
\begin{equation}
\Pr(N_{a,y} \leq \tau) \leq \exp(-2n(p_{A,\hat{Y}}(a,y)-\frac{ \tau}{n})) \longrightarrow_{n \rightarrow \infty}0.
\end{equation}
Therefore because $\vert \mathcal{A} \vert \vert \mathcal{Y} \vert$ is finite we have the consistency of $\UFIind[(q^*)]$.

\subsection{Some intuition on the impact of grouping protected attributes} \label{app.partitions}
Let $q$ a partition and $\rho$ a partition coarser than $q$, which means that every element of $q$ is a subset of some element of $\rho$. We want to somewhat relate the approximations and inequalities obtained using $\rho$ or $q$. Using the fact that $\rho$ is coarser than $q$, we can define for any $r\in \rho$ the partition $q_{r}=\{t \in q \mid t \subset r\}$ of $r$. This is a partition because the $t$ are disjoints, cover the whole set so $r$ as well in particular, and any $t \in q$ can either be a subset of $r$ or disjoint as $\rho$ is coarser than $q$.

We redefine the random variable $L$ for these partitions. We define $L^{(q)}=\log(p_A/\prod_{t \in q} p_{A_t})\circ A$ and $L^{(q,r)}=\log(p_{A_r}/\prod_{t \in q_r} p_{A_t}) \circ A$.

\begin{proposition}
We have
\begin{equation}
L^{(q)}=L^{(\rho)}+\sum_{r \in \rho \setminus q} L^{(q,r)}.
\end{equation}
Therefore 
\begin{equation}
\mathbb{E}[L^{(q)}]=\mathbb{E}[L^{(\rho)}]+\sum_{r \in \rho \setminus q} \mathbb{E}[L^{(q,r)}] \label{eq.part_mean}
\end{equation}
and 
\begin{equation}
\begin{split}
&\Var(L^{(q)})=\Var(L^{(\rho)})+\sum_{r \in \rho \setminus q } \Var(L^{(q,r)})+  \\ 
&2 \sum_{r \in \rho \setminus q} \! \Cov(L^{(\rho)},L^{(q,r)}) + \! \sum_{\substack{(r_1,r_2) \in (\rho \setminus q)^2 \\ r_1 \neq r_2}} \! \Cov(L^{(q,r_1)},L^{(q,r_2)}).
\end{split}
\end{equation}
\end{proposition}

\begin{proof}
Using the partitions $q_r$ we can group together the $p_{A_t}$ terms:
\begin{align*}
\frac{p_A(a)}{\prod_{t \in q}p_{A_t}(a_t)}&=\frac{p_A(a)}{\prod_{t \in q}p_{A_t}(a_t)}\frac{\prod_{r \in \rho} p_{A_r}(a_r)}{\prod_{r \in \rho} p_{A_r}(a_r)} \\
&=\frac{p_A(a)}{\prod_{r \in \rho} p_{A_r}(a_r)}\prod_{r \in \rho \setminus q} \frac{p_{A_r}(a_r)}{\prod_{t \in q_r} p_{A_t}(a_t)}.
\end{align*}
Then by taking the $\log$, we obtain the proposition. 
\end{proof}
Looking at \eqref{eq.part_mean}, we have the interesting property that if $\rho$ is a coarser partition than $q$, then $C(A^{(\rho)}) \leq C(A^{(q)})$. This corresponds to the intuition that taking coarser partition decreases some measure of independence, which is here the total correlation. 

We define $\ell_t=\log(p_{A_t}/\prod_{k \in t} p_{A_k})$ and $L_t=\ell_t(A_t)$. By applying the previous proposition, and by remarking that any partition is coarser than the set of all singletons we obtain the following corollary. 

\begin{corollary}
For any partition $q \in \mathcal{Q}$ we have
\begin{align}
L=L^{(q)}+\sum_{\substack{t \in q \\ \vert t \vert>1}} L_t
\end{align}
\end{corollary}

Which is why when using a partition $q$ to group together the protected attributes, we may be able to reduce the original variance of $L$ and $L_y$, hence reduce $\smax$ which is an increasing function of the variances. The decrease in $\smax$ is not guaranteed when using a coarser partition because of the covariance terms, but empirically this is often the case. 

\section{Additional Experiments and Plots} \label{app.experiments}

In this section we will present additional plots from the experiments conducted in Section \ref{sec.experiments}. 

The experiments were conducted on a machine with a i7 7700HQ CPU, and 8gb of ram. Running all the experiments took about 1 full day. 

The main dataset used is a publicly available sample from the 1990 US census. The US census is legally mandated, hence every citizen has to give its information to the US government. No identifiable information is available, and the samples were randomly chosen from the original full dataset. Full information is available at the \hyperlink{https://archive.ics.uci.edu/ml/datasets/US+Census+Data+(1990)}{UCI archive link}.

We reproduce here Figure \ref{fig.comp_conv} and Figure \ref{fig.exp_part} on Figure \ref{fig.comp_conv_CI} and Figure \ref{fig.exp_part_CI} adding the $1^{\text{st}}$ and $10^{\text{th}}$ decile but only using $\alpha=1$ for $\UFIB$ and $\tau=10$ to make it readable. 

\begin{figure}[ht] 
\begin{center} 
\centerline{\includegraphics[width=1.3\columnwidth]{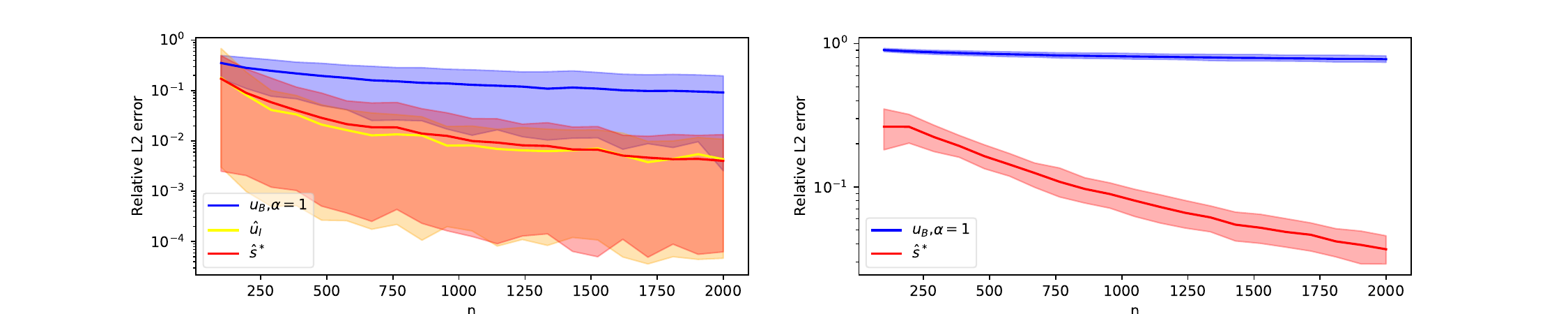}} 
\caption{Average $L_2^r$ convergence rate, on real data for the left one, and synthetic data for the right one. In all these graphs the intervals represent the $1^{\text{st}}$ and $10^{\text{th}}$ decile.}\label{fig.comp_conv_CI}
\end{center}
\end{figure}

\begin{figure}[ht] 
\begin{center}
\centerline{\includegraphics[width=1.3\columnwidth]{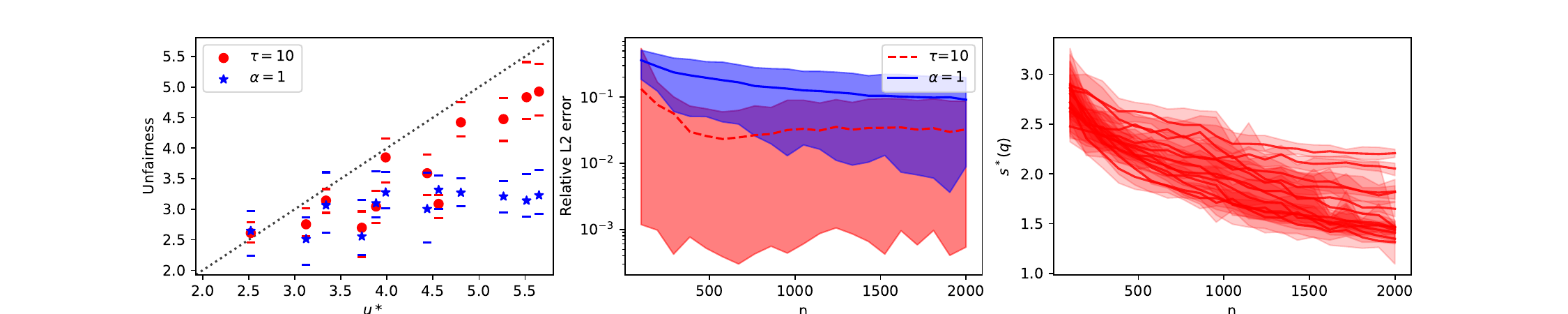}} 
\caption{In the left-most plot each point with same shape and color correspond a different $D_i$, with the estimators values taken at $n=2000$. The middle plot is the average $L_2^r$ error rate on the real data. The rightmost plot is the average evolution of $s^*(q^*)$ for $\tau=10$ as $n$ increases. In all these graphs the intervals represent the $1^{\text{st}}$ and $10^{\text{th}}$ decile.} \label{fig.exp_part_CI}
\end{center}
\end{figure}

We recall that we always take $\delta=0.1$. We present on Figure \ref{fig.chernoff_error} a comparison of the relative error rate between $\UFIB$, $\smax$, and $\inf_{g(\boldsymbol{\lambda})} \lambda^+-\lambda^-$ where $g$ is estimated through the empirical distribution, and the optimization problem is solved numerically. We see that while it is easier to estimate than $\UFIB$, it is harder than $\smax$. Note that numerically solving the minimization problem may lead to numerical errors for too low number of samples.

\begin{figure}[ht] 
\begin{center} 
\centerline{\includegraphics[width=1.3\columnwidth]{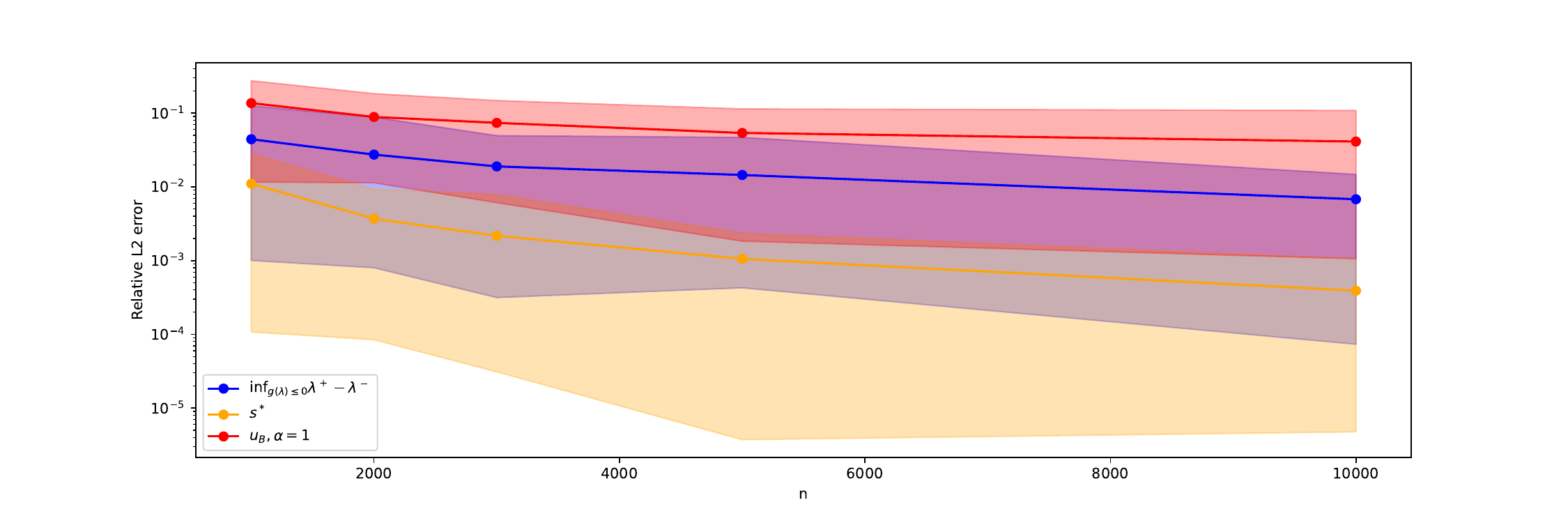}} 
\caption{Average $L_2^r$ convergence rate, on real data. The intervals represent the $1^{\text{st}}$ and $10^{\text{th}}$ decile.} \label{fig.chernoff_error}
\end{center}
\end{figure}

We also compare some of the bounds presented throughout this paper. Here we do not care about their estimation, but only their asymptotic value. In addition, we want to evaluate the impact of using partitions on these bounds. In order to do so we take a sample of size $n=2000$ of each of our $D_i$, and compute $q^*$. Then we use the full dataset to compute $s^*(q^*)$ for $\tau=10$, $\gamma(q^*)$ and $\inf_{g(\boldsymbol{\lambda})} \lambda^+-\lambda^-$. We also compute the exact unfairness quantile $\epsilon^*(\delta)$. We obtain Figure \ref{fig.various_bounds}. We see that using partitions seem to always yield tighter bounds, and that most of the time $\epsilon_1\geq \epsilon_2 \geq \epsilon_3$. Even the improved bounds are still far from $\epsilon^*$ (the optimal bound in probability), but it shows that these bounds can be improved. We conjecture that if we want to find reliable information on $\UFI$ when $d$ becomes very large, these bounds can be useful in practice. Conversely, these bounds and approximations should not be used if sufficient information is available to directly use $\UFIB$ (for instance at least $1$ sample by protected group). 

\begin{figure}[ht] 
\begin{center}
\centerline{\includegraphics[width=1.3\columnwidth]{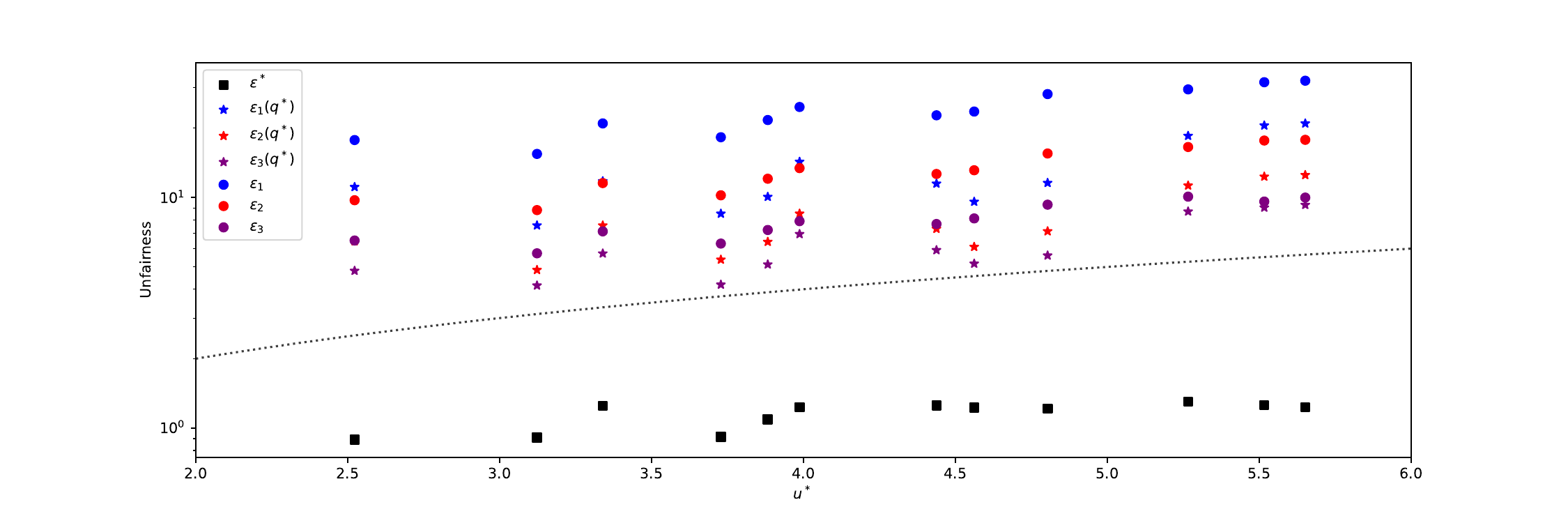}}  
\caption{Exact computation of the various bounds for each of the 12 $D_i$ selected.} \label{fig.various_bounds}
\end{center}
\end{figure}

\end{document}